\documentclass{article}
\usepackage[utf8]{inputenc}
\usepackage[T1]{fontenc}
\usepackage{times}
\usepackage{helvet}
\usepackage{courier}

\usepackage{amsmath}
\usepackage{amssymb}
\usepackage{amsfonts}

\usepackage{booktabs}
\usepackage{multirow}
\usepackage{makecell}
\usepackage{array}
\usepackage{tabularx}

\usepackage{graphicx}
\usepackage{xcolor}
\usepackage{caption}

\usepackage{algorithm}
\usepackage{algorithmic}

\usepackage{enumitem}
\usepackage{tcolorbox}
\usepackage{float}
\usepackage{nicefrac}
\usepackage{microtype}

\usepackage{url}
\usepackage[colorlinks=true, linkcolor=blue, citecolor=blue, urlcolor=blue]{hyperref}

\usepackage{tikz}

\usepackage[numbers,square]{natbib}

\usepackage{arxiv}

\usepackage{amsthm}

\newtheorem{theorem}{Theorem}[section]
\newtheorem{lemma}[theorem]{Lemma}
\newtheorem{proposition}[theorem]{Proposition}

\title{NeuroSymActive: Differentiable Neural-Symbolic Reasoning with Active Exploration for Knowledge Graph Question Answering}

\author{
    Rong Fu \\
    Independent Researcher \\
    Corresponding author \and
    Yang Li \\
    Independent Researcher \and
    Zeyu Zhang \\
    Independent Researcher \and
    Jiekai Wu \\
    Independent Researcher \and
    Yaohua Liu \\
    Independent Researcher \and
    Shuaishuai Cao \\
    Independent Researcher \and
    Yangchen Zeng \\
    Independent Researcher \and
    Yuhang Zhang \\
    Independent Researcher \and
    Xiaojing Du \\
    Independent Researcher \and
    Simon Fong \\
    Independent Researcher
}

\renewcommand{\headeright}{}
\renewcommand{\undertitle}{}
\renewcommand{\shorttitle}{NeuroSymActive}

\hypersetup{
    pdftitle={NeuroSymActive: Differentiable Neural-Symbolic Reasoning with Active Exploration for Knowledge Graph Question Answering},
    pdfsubject={Machine Learning, Artificial Intelligence},
    pdfauthor={Rong Fu, Yang Li, Zeyu Zhang, Jiekai Wu, Yaohua Liu, Shuaishuai Cao, Yangchen Zeng, Yuhang Zhang, Xiaojing Du, Simon Fong},
    pdfkeywords={Knowledge Graph Question Answering, neural-symbolic reasoning, differentiable logic, active exploration, Monte Carlo search},
    pdfstartview={FitH},
    colorlinks=true,
    linkcolor=red,
    citecolor=green,
    filecolor=magenta,
    urlcolor=cyan
}

\begin{document}
\maketitle

\begin{abstract}
Large pretrained language models and neural reasoning systems have advanced many natural language tasks, yet they remain challenged by knowledge-intensive queries that require precise, structured multi-hop inference. Knowledge graphs provide a compact symbolic substrate for factual grounding, but integrating graph structure with neural models is nontrivial: naively embedding graph facts into prompts leads to inefficiency and fragility, while purely symbolic or search-heavy approaches can be costly in retrievals and lack gradient-based refinement. We introduce \textbf{NeuroSymActive}, a modular framework that combines a differentiable neural-symbolic reasoning layer with an active, value-guided exploration controller for Knowledge Graph Question Answering. The method couples soft-unification style symbolic modules with a neural path evaluator and a Monte-Carlo style exploration policy that prioritizes high-value path expansions. Empirical results on standard KGQA benchmarks show that NeuroSymActive attains strong answer accuracy while reducing the number of expensive graph lookups and model calls compared to common retrieval-augmented baselines.
\end{abstract}


\section{Introduction}
Large language models have produced striking improvements in many language tasks, but they do not always provide robust solutions for questions that require structured, multi-hop reasoning over world knowledge. Knowledge graphs are a natural vehicle for encoding factual relations and multi-step semantic chains, yet incorporating graphs into neural pipelines raises several practical and methodological issues. One common strategy retrieves relevant subgraphs and injects them into neural models or prompts; this approach grounds model predictions but can be inefficient, since iterative retrieval and repeated model calls are often necessary to assemble multi-hop evidence \cite{jiang2022retrieval,tang2024self,guo2024lightrag}. Another line of work embeds graph structure into differentiable modules or applies soft logical operators to enable gradient-based refinement; these methods improve interpretability and allow end-to-end training, but they can struggle to discover long-range compositional paths without targeted search guidance \cite{gao2022learning,maene2023soft,bai2023complex}. Symbolic search and beam-style expansion offer exhaustive path coverage but at the cost of many KG lookups and brittle heuristics \cite{atif2023beamqa,li2022smartquery}. Recent attempts to combine learning and search point to the benefit of guiding exploration with learned value estimates, but existing instantiations have not fully reconciled differentiable symbolic constraints with principled active exploration \cite{chen2023neuroescape,zhang2023critical}.

Motivated by these gaps, we present NeuroSymActive, a framework that tightly integrates three components. The first is a differentiable neural-symbolic reasoning layer that supports soft pattern matching and differentiable rule scoring, which allows logical preferences to be learned from data while retaining interpretability. The second is a neural path evaluator that assigns value estimates to partial reasoning trajectories. The third is an active exploration controller that uses the evaluator to prioritize expansion and rollout, thereby focusing retrieval effort on promising branches and reducing unnecessary KG queries. This design intentionally separates offline learning of symbolic affinities from online search control, enabling efficient and adaptive multi-hop reasoning under constrained retrieval budgets.

Our contributions are as follows. Our first contribution is a compact differentiable neural-symbolic module for KG reasoning that supports soft-unification, differentiable scoring of rule-like patterns, and seamless interfacing with neural path evaluators. Our second contribution is an active exploration mechanism that integrates learned value signals into a Monte Carlo style search procedure; this mechanism concentrates retrieval and expansion on high-payoff paths and reduces the number of costly lookups. Our third contribution is an empirical study demonstrating that the combined system improves the accuracy-efficiency trade-off on standard KGQA benchmarks compared to retrieval-heavy baselines and to purely symbolic pipelines. Together, these elements make multi-hop KGQA more tractable for deployments that must minimize query cost while preserving interpretable reasoning traces.

\section{Related Work}
\subsection{Neural--Symbolic Integration}
Neural and symbolic methods have been combined to yield systems that leverage the perceptual strength of deep networks alongside the compositionality and interpretability of symbolic rules. Early work on joint neural perception and logic reasoning established the value of end-to-end coupling between perception modules and symbolic inference engines \cite{duan2022deeplogic,nayyeri2021logicenn,pryor2022neupsl}. Differentiable relaxations of logic programs enable gradient-based learning of rule weights and structure, which improves robustness on noisy or structured examples \cite{shindo2021differentiable,shindo2023alpha,gao2022learning,yue2024learning}. More recent studies scale differentiable neuro-symbolic reasoning to large knowledge graphs and propose practical approximations that trade off efficiency and fidelity to logical constraints \cite{shengyuan2023differentiable,zhang2024mote}. Surveys and reviews summarize advances and identify central challenges such as scalability, multimodal integration, and preserving interpretability \cite{bhuyan2024neuro,nawaz2025review}. These works provide the foundational techniques that our system adapts to the multi-hop KGQA setting.

\subsection{Differentiable Rule Learning and Logic Networks}
Learning first-order rules through differentiable semantics has seen rapid progress. Approaches construct continuous surrogates for logical operators, enabling rules to be discovered and refined with gradient signals \cite{gao2022learning,shindo2021differentiable,yue2024learning}. Architectures that produce interpretable logic networks by softening discrete components have demonstrated competitive performance while improving model transparency \cite{yue2024learning,shindo2023alpha}. Complementary lines of work study probabilistic soft logic and energy-based neuro-symbolic models that meld weighted rules with embedding-based scores \cite{pryor2022neupsl,shengyuan2023differentiable}. These methods inform our differentiable inductive logic layer and the techniques used to update rule confidences from supervised and human-provided signals.

\subsection{Human-in-the-Loop and Active Supervision}
Human involvement during model training and inference can dramatically increase sample efficiency and correctness for uncertain decisions. Active learning and related human-in-the-loop paradigms provide mechanisms to query annotators selectively, with proven benefits across domains \cite{mosqueira2023human,huang2024efficient,kim2025active,kim2025freeson}. The literature distinguishes active selection strategies that reduce label cost from interactive protocols that allow richer human guidance \cite{mosqueira2023human,hejna2023few}. Graph-specific active schemes, including hybrid uncertainty reduction and noisy-crowd models, are particularly relevant when queries concern graph structure or relation validity \cite{li2022smartquery,zhang2024nc}. We design our query policy and replay buffer to concentrate human effort where it yields maximal information gain, building on these established techniques.

\subsection{Uncertainty Quantification and Entropy Prediction}
Accurate uncertainty estimates are central to effective query selection and to adaptive exploration strategies. Work on epistemic and heteroscedastic uncertainty for neural networks proposes principled estimators and Bayesian approximations that improve active-selection reliability \cite{hein2022comparison,wang2024epistemic,immer2023effective,seitzer2022pitfalls,rathnakumar2024bayesian}. For graph-structured data and few-shot KG completion, uncertainty-aware GNNs have been proposed to model distributional uncertainty and to guide sampling \cite{li2024uncertainty,ziatdinov2024active,tzes2022graph}. These techniques motivate our neural entropy predictor and the calibration methods we apply to the uncertainty module.

\subsection{Retrieval-Augmented Generation and RAG-DDR}
Integrating retrieval components with generative models has become a practical approach to reduce model hallucinations and recover up-to-date facts \cite{li2024rag,wang2024knowledge,ji2024retrieval,li2024simple}. Recent proposals frame end-to-end optimization of retrieval-augmented systems through differentiable reward signals, aligning retrieval and generation modules for global performance gains \cite{li2024rag,gao2025d,wang2024knowledge}. Our adapter design and the way we condition LLMs on compact, fused path representations are informed by these retrieval and optimization strategies \cite{li2024rag,gao2025d,ji2024retrieval}.

\subsection{Search, Monte Carlo Tree Search and Differentiable Exploration}
Search-based methods that view multi-hop KGQA as sequential decision making have advanced with MCTS- and RL-style explorations. Training-free tree search guided by LLM evaluations yields competitive reasoning traces \cite{song2025rekg,gao2024interpretable,yu2024exact}. Differentiable relaxations for discrete selection, such as Gumbel-Softmax and straight-through estimators, enable backpropagation through exploration and selection decisions \cite{shah2024improving,chaudhary2023gumbel}. Prior work combines structured exploration with neural guidance to improve path discovery over KGs \cite{jiang2023path,zhao2024rpr,dong2023hierarchy}. We adopt progressive widening and a continuously relaxed rollout mechanism so that inner-loop search remains differentiable and compatible with gradient-based updates. Beyond knowledge graph question answering, recent work has explored Monte Carlo Tree Search and evolutionary search for tool planning and control in LLM agents, with ToolTree proposing a dual-feedback MCTS framework with bidirectional pruning to improve planning efficiency in tool-use scenarios~\cite{yang2026tooltree}, and EvoTool investigating self-evolving tool-use policies via blame-aware mutation and diversity-aware selection to enhance agent control~\cite{yang2026evotool}.

\subsection{Knowledge Graph Question Answering and Multi-hop Reasoning}
Classical and modern KGQA approaches range from end-to-end embedding models to hybrid pipelines that retrieve subgraphs and apply symbolic or neural reasoning \cite{zhang2022subgraph,jiang2023path,zheng2023multi,dong2023hierarchy}. The multi-hop KGQA literature emphasizes accurate subgraph retrieval, robust path scoring, and methods for mitigating spurious paths \cite{sun2019pullnet,sun2018open,shi2021transfernet,jiang2023path}. Recent works explore LLMs as reasoning agents over KGs and study decomposition or conversation-based interactions to refine answers \cite{xiong2024interactive,yixing2024chain,omar2023chatgpt}. Our framework addresses the three core KGQA challenges simultaneously: targeted retrieval, symbolic plausibility scoring, and LLM-conditioned answer generation.

\subsection{Prompting, Adapter Strategies and Compact LLMs}
Prompt-based and adapter-based conditioning of frozen LLMs provide parameter-efficient ways to incorporate structured context \cite{bai2023knowprefix,wu2023infoprompt,chen2024lifelong,cheng2023uprise}. Knowledge adapters map fused graph and symbolic signals into compact soft prompts that preserve task-relevant information while controlling token overhead. These design choices are aligned with findings that compact prompt tuning can make LLMs effectively consume external knowledge without full fine-tuning \cite{wu2023infoprompt,bai2023knowprefix,chen2024lifelong}.

\subsection{Complementary Methods: Rule-Regularized Embeddings and Scalability}
A number of works enhance embedding-based reasoning with logical constraints or rule regularizers to improve interpretability and generalization \cite{nayyeri2021logicenn,pryor2022neupsl,zhang2024mote}. Scalability-oriented research studies approximations and filtering mechanisms that reduce the combinatorial cost of rule evaluation on large graphs \cite{shengyuan2023differentiable,zhang2024mote,wan2025automatic}. These contributions guided our choices for selective triple filtering and for efficient batched evaluation during search.

\subsection{Summary of Positioning}
The contributions described in this paper draw upon and synthesize multiple active research threads: differentiable logic and rule learning \cite{shindo2021differentiable,gao2022learning,yue2024learning,shindo2023alpha}, neural-symbolic fusion and soft-rule integration \cite{duan2022deeplogic,pryor2022neupsl,shengyuan2023differentiable}, human-in-the-loop and active supervision strategies \cite{mosqueira2023human,li2022smartquery,huang2024efficient}, uncertainty modeling for selective querying \cite{hein2022comparison,li2024uncertainty,immer2023effective}, and differentiable exploration mechanisms that allow end-to-end updates through search \cite{song2025rekg,shah2024improving,gao2024interpretable}. The following sections explain how these elements are combined and extended in our NeuroSymActive framework to address multi-hop KGQA with constrained annotation budgets.
\begin{figure*}[t]
  \centering
  \includegraphics[width=0.88\textwidth]{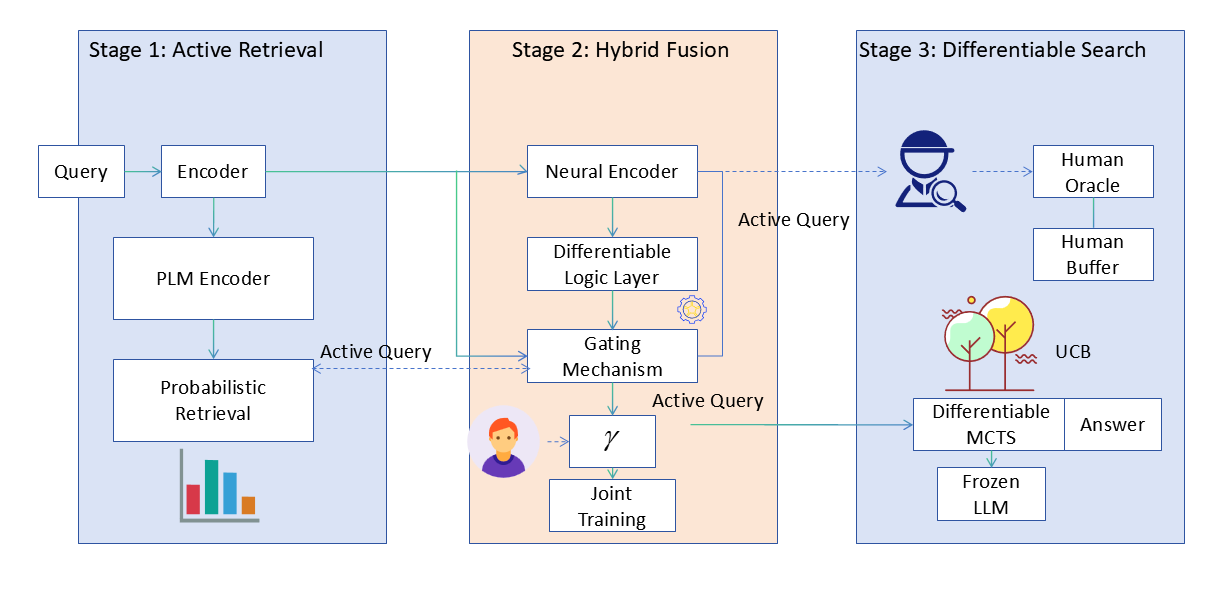} 
  \caption{Architectural overview of the \textbf{NeuroSymActive} framework for knowledge graph question answering. 
  The framework operates via a coupled dual-loop optimization process across three main stages: 
  \textbf{Stage 1: Uncertainty-Aware Active Retrieval}, which utilizes a Bayesian head to model heteroscedastic uncertainty in hop prediction and a neural entropy predictor $\eta_\theta$ to estimate information gain ($IG$). This stage selectively invokes a \textbf{Human Oracle} via the $\mathrm{QUERY\_HUMAN}$ action when uncertainty $\mathcal{U}(\mathbf{s}_t)$ exceeds the threshold $\tau_{\mathrm{hop}}$.
  \textbf{Stage 2: Differentiable Neural-Symbolic Fusion}, where a hybrid Knowledge Adapter merges continuous path embeddings ($\mathbf{z}_f^{\mathrm{neural}}$) with soft symbolic plausibility scores ($\mathbf{z}_f^{\mathrm{sym}}$) derived from a \textbf{Differentiable Inductive Logic Layer (DILL)}. Rule confidences $w_\rho$ are updated via gradient signals from active supervision stored in the human replay buffer $\mathcal{D}_{\mathrm{human}}$.
  \textbf{Stage 3: Differentiable MCTS with Active Exploration}, which implements a relaxed Monte Carlo Tree Search. It employs \textbf{Progressive Widening} governed by predictive uncertainty and treats human intervention as explicit nodes with value $V_{\mathrm{human}}(s)$. The search is fully differentiable via Gumbel-Softmax relaxations, allowing end-to-end joint training of the policy $\pi_\phi$, value $v_\psi$, and symbolic weights through a composite multi-objective loss $\mathcal{L}_{\mathrm{total}}$.} 
  \label{fig:neurosymactive_architecture}
\end{figure*}
\section{Methodology}

This section describes the NeuroSymActive methodology for multi-hop question answering over knowledge graphs. The design preserves a retrieve--embed--reason pipeline while introducing active, differentiable retrieval, a hybrid neural-symbolic knowledge adapter, and an iterative neural-guided graph exploration mechanism that interleaves retrieval and reasoning.

\subsection{Problem formalization}
We cast knowledge-graph question answering as a partially observable sequential decision process with an explicit neural-symbolic duality. Let \(\mathcal{G}=\{(h,r,t)\}\) denote a knowledge graph and let \(q\) be a natural-language query. The agent maintains an internal cognitive state \(\mathbf{s}_t\) that decomposes into a continuous neural component \(\mathbf{s}_t^{\mathrm{neural}}\) and a discrete symbolic component \(\mathbf{s}_t^{\mathrm{sym}}\). The neural component contains differentiable embeddings of currently retrieved subgraphs and related latent variables, whereas the symbolic component encodes discrete logical constraints, type relations and other symbolic checks used for pruning and consistency verification.

At each decision step \(t\) the agent selects an action \(a_t\) from a set of candidate retrieval/expansion operations \(\mathcal{A}_t\) augmented with an explicit human-query action \(\mathrm{QUERY\_HUMAN}\), that is \(a_t\in\mathcal{A}_t\cup\{\mathrm{QUERY\_HUMAN}\}\). Selecting \(\mathrm{QUERY\_HUMAN}\) triggers a supervised annotation with incurred cost \(c_{\mathrm{human}}>0\). An episode terminates when the agent emits a final answer \(a\) or exhausts its computational or annotation budget.

The agent's objective trades off answer correctness and human annotation cost: it seeks to maximize expected accuracy while minimizing cumulative annotation expenditure. Formally we optimize a bi-objective that can be collapsed into a single scalar objective by penalizing human queries:
\begin{equation}\label{eq:obj}
\begin{aligned}
\mathcal{J}&=\mathbb{E}\bigl[\mathbb{I}[\hat{a}=a^\star]\bigr]\\
&\quad -\beta\sum_{t}\mathbb{I}\bigl[a_t=\mathrm{QUERY\_HUMAN}\bigr]\,c_{\mathrm{human}},
\end{aligned}
\end{equation}
where \(\hat{a}\) denotes the emitted answer, \(a^\star\) is the ground-truth answer, \(\mathbb{I}[\cdot]\) is the indicator function, and \(\beta>0\) balances accuracy and annotation cost. In (\ref{eq:obj}), the expectation is taken over the agent's stochastic policy and any environment randomness.

A reasoning path \(p\) is represented by an ordered sequence of triples with both neural embeddings and symbolic identifiers:
\begin{equation}\label{eq:path_def}
p=\bigl((e_0,r_1,e_1),\dots,(e_{L-1},r_L,e_L)\bigr),
\end{equation}
where each entity \(e_i\) and relation \(r_j\) is associated with a continuous embedding \(\mathbf{e}_i,\mathbf{r}_j\in\mathbb{R}^d\) used by the differentiable modules and with a discrete symbolic identifier used by the symbolic checks, and \(L\) denotes the dynamic path length determined by the agent's value function or budget termination. In (\ref{eq:path_def}), the neural embedding vectors live in \(\mathbb{R}^d\) and the symbolic identifiers index logical constraints and types.

\subsection{Algorithm framework}
The following algorithm summarizes the coupled inner/outer loop of NeuroSymActive in a concise form. Key estimators and relaxations are referenced where used. The overall procedure is summarized in Algorithm~\ref{alg:neurosymactive_compact}.

\begin{algorithm}[t]
\caption{NeuroSymActive}
\label{alg:neurosymactive_compact}
\begin{algorithmic}[1]
\REQUIRE knowledge graph \(G\), query \(q\), budgets, human buffer \(\mathcal{D}_{\mathrm{human}}\)
\STATE extract anchors \(B\) and compute \(V_q=\mathrm{PLM}(q)\)
\STATE obtain hop distribution and uncertainty via (\ref{eq:bayes_hq})
\STATE initialize \(G_{\mathrm{cur}}\leftarrow\varnothing\) and search root
\WHILE{not terminated}
  \STATE if uncertainty \(\mathcal{U}(\mathbf{s}_t)>\tau_{\mathrm{human}}\) then invoke \(\mathrm{QueryHuman}\) and update rules by (\ref{eq:rule_update})
  \FOR{iteration = 1 \TO \(N_{\mathrm{inner}}\)}
    \STATE enumerate candidate relations \(\{r\}\) and predict \(\widehat{\mathrm{IG}}(r)\) via (\ref{eq:neural_ig})
    \STATE compute utility \(u(r)\) using (\ref{eq:active_util}) and obtain soft weights \(\{s_r\}\) via (\ref{eq:gumbel_softmax})
    \STATE expand \(G_{\mathrm{cur}}\) with top/sampled \(s_r\), compute \(\mathbf{z}_f^{\mathrm{neural}}\) via (\ref{eq:zf_neural})
    \STATE compute symbolic activations and score via (\ref{eq:soft_rule})--(\ref{eq:sym_score}) and fuse using (\ref{eq:fusion_stage2})
    \STATE run differentiable rollouts under tree policy \(\pi_{\mathrm{tree}}\) in (\ref{eq:soft_tree}), guided by \(\mathrm{UCB}_{\mathrm{PW}}\) in (\ref{eq:ucb_pw}); treat human choices with value in (\ref{eq:human_node})
  \ENDFOR
  \STATE form \(\mathcal{L}_{\mathrm{total}}\) per (\ref{eq:total_loss}) and update \(\pi_\phi, v_\psi,\) adapter and DILL weights
\ENDWHILE
\RETURN highest-scoring answer from the LLM conditioned on the adapter prompt
\end{algorithmic}
\end{algorithm}

\paragraph{Implementation guidance.}  
Use minibatched inner-loop rollouts with shared batched GNNs for \(\eta_\theta\) to amortize entropy predictions. Maintain separate optimizers for DILL parameters \(\{w_\rho\}\) and neural encoder weights to allow differential learning rates. For stability, anneal the Gumbel temperature \(\tau\) in (\ref{eq:gumbel_softmax}) and the search temperature \(\tau_{\mathrm{search}}\) in (\ref{eq:soft_tree}), and apply gradient normalization when minimizing (\ref{eq:total_loss}) so that no single objective dominates training.

\subsection{High-level architecture: coupled dual-loop optimization}
NeuroSymActive operates through two tightly coupled loops that run on distinct time scales and exchange information continuously. The inner loop is a fast, differentiable neural-symbolic exploration driven by a neural policy \(\pi_\phi\) and a value estimator \(v_\psi\). At this time scale the system performs many lightweight expansions: the policy proposes candidate relation expansions, symbolic constraints encoded as a differentiable logical layer prune infeasible expansions, and discrete choices are represented with continuous relaxations (for example Gumbel-Softmax) so that gradients can flow through the exploration decisions. The inner loop thus produces a relaxed reasoning graph whose edge-selection variables are stochastic, differentiable weights.

The outer loop runs at a slower time scale and implements active human-in-the-loop supervision. An uncertainty quantifier \(\mathcal{U}(\mathbf{s}_t)\), which monitors entropy and saturation of information gain, decides whether to invoke \(\mathrm{QUERY\_HUMAN}\). When the uncertainty metric exceeds a pre-specified threshold or when information gain plateaus, the system requests expert feedback for critical decisions such as hop depth, relation relevance, or path validity. Human annotations are stored in a replay buffer \(\mathcal{D}_{\mathrm{human}}\) and used for off-policy updates of \(\pi_\phi\), \(v_\psi\), and the neural-symbolic adapter.

The knowledge adapter forms the bridge between modalities by fusing three signals: continuous path embeddings produced by the neural encoders, soft plausibility scores returned by the differentiable symbolic module, and uncertainty signals used to trigger outer-loop queries. The adapter projects fused representations into the frozen LLM's token embedding space as compact soft prompts. Gradients from the LLM's downstream loss are routed back to the differentiable components through the adapter's continuous interfaces while symbolic constraints remain enforced via soft scores.

\subsection{Stage 1: Uncertainty-aware active retrieval}
We first refine hop prediction using a probabilistic neural head that models heteroscedastic uncertainty. The semantic encoder computes
\begin{equation}\label{eq:Vq_def}
V_q=\mathrm{PLM}(q),
\end{equation}
where \(\mathrm{PLM}(\cdot)\) is a pretrained language encoder and \(V_q\in\mathbb{R}^d\) is the semantic vector for query \(q\). A Bayesian head yields a distribution over hop budgets:
\begin{equation}\label{eq:bayes_hq}
\begin{aligned}
P(h_q\mid V_q)&=\mathrm{Softmax}\bigl(\mathrm{MLP}(V_q)+\boldsymbol{\epsilon}\bigr),\\
&\quad\text{where }\boldsymbol{\epsilon}\sim\mathcal{N}\bigl(\mathbf{0},\sigma^2(V_q)\mathbf{I}\bigr).
\end{aligned}
\end{equation}
where \(\mathrm{MLP}(\cdot)\) is a small feedforward network and \(\sigma^2(V_q)\) is a learned heteroscedastic variance. In (\ref{eq:bayes_hq}), the additive Gaussian noise models epistemic/aleatoric uncertainty and the softmax converts logits to a probability distribution over discrete hop counts. If \(\max_h P(h_q=h\mid V_q)\) falls below a threshold \(\tau_{\mathrm{hop}}\), the system issues \(\mathrm{QueryHuman}(\textit{hop-depth},q)\) to obtain an annotated hop count \(h_q^\star\).

Estimating the information gain for a candidate relation \(r\) is computationally expensive if performed by Monte-Carlo simulation. We therefore learn a neural entropy predictor \(\eta_\theta\) that maps subgraph features to predicted answer entropy and use it to approximate information gain:
\begin{equation}\label{eq:neural_ig}
\widehat{\mathrm{IG}}(r)=\eta_\theta\bigl(G_{\mathrm{cur}},r,V_q\bigr)-\eta_\theta\bigl(G_{\mathrm{cur}},\varnothing,V_q\bigr),
\end{equation}
where \(\eta_\theta(\cdot)\) is implemented as a graph neural network that ingests the current subgraph \(G_{\mathrm{cur}}\), the candidate relation \(r\), and the query embedding \(V_q\), and returns a scalar predictive entropy. In (\ref{eq:neural_ig}), the second term denotes predicted entropy without adding relation \(r\).

We generalize the acquisition utility to support both autonomous and human-augmented selection. The two-mode utility is
\begin{equation}\label{eq:active_util}
u(r)=
\left\{
\begin{array}{ll}
\widehat{\mathrm{IG}}(r)-\lambda\,C(r), & \text{(auto mode)}\\[6pt]
\widehat{\mathrm{IG}}_{\mathrm{human}}(r)-\lambda\,C(r)-c_{\mathrm{human}}, & \text{(human)}
\end{array}
\right.
\end{equation}
where \(C(r)\) is an estimated retrieval cost, \(\lambda>0\) balances information gain and cost, \(c_{\mathrm{human}}\) is the explicit annotation cost for invoking human feedback, and \(\widehat{\mathrm{IG}}_{\mathrm{human}}(\cdot)\) leverages relevance labels from the human replay buffer \(\mathcal{D}_{\mathrm{human}}\). In (\ref{eq:active_util}), the human-augmented mode is considered when the uncertainty quantifier or policy indicates high ambiguity at a decision point.

To allow end-to-end optimization, discrete relation selection is relaxed via a Gumbel-Softmax reparameterization. For a candidate set \(\{r\}\) with base logits \(\log\pi(r)\) we compute soft selection weights
\begin{equation}\label{eq:gumbel_softmax}
\begin{aligned}
s_r&=\frac{\exp\bigl((\log\pi(r)+g_r)/\tau\bigr)}{\sum_{r'}\exp\bigl((\log\pi(r')+g_{r'})/\tau\bigr)},\\
&\qquad g_r\sim\mathrm{Gumbel}(0,1),
\end{aligned}
\end{equation}
where \(\tau>0\) is a temperature parameter and \(s_r\in[0,1]\) approximates a one-hot selection in the low-temperature limit. In (\ref{eq:gumbel_softmax}), sampling the Gumbel noise \(g_r\) and dividing by \(\tau\) yields a differentiable approximation to categorical sampling used during inner-loop exploration and training.

\subsection{Stage 2: Differentiable neural-symbolic fusion}

The knowledge adapter places a differentiable inductive logic layer at the core of neural-symbolic integration so that symbolic inference participates in end-to-end gradient-based learning. The adapter receives path-level neural encodings and yields fused representations that carry both distributional semantics and logical validity.

For a sampled reasoning path \(p\) we first obtain element-wise embeddings from textual and structural encoders. For the \(i\)-th triple in \(p\) denote the relation, head entity and tail entity embeddings by
\begin{equation}\label{eq:embeds_stage2}
\begin{aligned}
\mathbf{e}^r_i&=\mathrm{Embed}(r_i),\quad
\mathbf{e}^h_i=\mathrm{Embed}(e_{i-1}),\\
&\mathbf{e}^t_i=\mathrm{Embed}(e_i),
\end{aligned}
\end{equation}
where \(\mathrm{Embed}(\cdot)\) denotes a shared textual encoder and each vector lies in \(\mathbb{R}^d\). In (\ref{eq:embeds_stage2}), the symbols \(\mathbf{e}^r_i\), \(\mathbf{e}^h_i\) and \(\mathbf{e}^t_i\) denote relation, head and tail continuous embeddings respectively.

Local triple structure is summarized by a structural encoder and aggregated across the path to obtain a neural structural summary:
\begin{equation}\label{eq:struct_stage2}
\mathbf{s}_i=\mathrm{StructEmb}\bigl(\mathbf{e}^h_i,\mathbf{e}^r_i,\mathbf{e}^t_i\bigr),\qquad
\mathbf{z}_s=\mathrm{Agg}\bigl(\mathbf{s}_1,\dots,\mathbf{s}_L\bigr),
\end{equation}
where \(\mathrm{StructEmb}(\cdot)\) is a small permutation sensitive network that captures local ordering and \(\mathrm{Agg}(\cdot)\) denotes a learnable pooling operator producing \(\mathbf{z}_s\in\mathbb{R}^d\). In (\ref{eq:struct_stage2}), \(\mathbf{s}_i\) and \(\mathbf{z}_s\) are the local and global structural summaries respectively.

A text fusion operator produces a path-level textual vector:
\begin{equation}\label{eq:zt_stage2}
\mathbf{z}_t=\mathrm{Fuse}\bigl(\mathbf{e}^h_1,\dots,\mathbf{e}^h_L,\mathbf{e}^r_1,\dots,\mathbf{e}^t_L\bigr),
\end{equation}
where \(\mathrm{Fuse}(\cdot)\) preserves hop order and returns \(\mathbf{z}_t\in\mathbb{R}^d\). In (\ref{eq:zt_stage2}), \(\mathbf{z}_t\) aggregates textual signals across the path.

The neural pathway produces a compact neural path vector:
\begin{equation}\label{eq:zf_neural}
\mathbf{z}_f^{\mathrm{neural}}=\mathrm{KnowledgeEncoder}\bigl([\mathbf{z}_t,\mathbf{z}_s]\bigr),
\end{equation}
where \(\mathrm{KnowledgeEncoder}(\cdot)\) denotes a parameterized neural network and \([\cdot,\cdot]\) denotes concatenation. In (\ref{eq:zf_neural}), \(\mathbf{z}_f^{\mathrm{neural}}\in\mathbb{R}^d\) is the neural encoding of path \(p\).

Symbolic inference is implemented as a differentiable inductive logic layer (DILL) that grounds learned rules into continuous predicate evaluations. Each rule \(\rho\) has a learned confidence weight \(w_\rho\in[0,1]\) and the rule body is evaluated by a neural predicate grounding \(\pi_\rho(p)\) that maps path embeddings to truth degrees. The soft activation of rule \(\rho\) on path \(p\) is
\begin{equation}\label{eq:soft_rule}
\phi_\rho(p)=T\bigl(w_\rho,\;\pi_\rho(p)\bigr),
\end{equation}
where \(T\) is a differentiable t-norm such as the product operator and \(\pi_\rho(p)\in[0,1]\) is the neural predicate grounding for \(\rho\) on path \(p\). In (\ref{eq:soft_rule}), \(\phi_\rho(p)\) denotes the soft truth value of rule \(\rho\) on \(p\).

Symbolic plausibility aggregates contributions from applicable rules to produce a path-level symbolic score:
\begin{equation}\label{eq:sym_score}
\mathbf{z}_f^{\mathrm{sym}}=\mathrm{Agg}_{\rho\in\mathcal{R}_p}\,\phi_\rho(p),
\end{equation}
where \(\mathcal{R}_p\) is the set of rules relevant to path \(p\) and the aggregation operator produces \(\mathbf{z}_f^{\mathrm{sym}}\in\mathbb{R}^d\). In (\ref{eq:sym_score}), \(\mathbf{z}_f^{\mathrm{sym}}\) encodes symbolic plausibility derived from the rule base.

Neural and symbolic vectors are fused with a learned gating mechanism that admits gradient flow into both pathways:
\begin{equation}\label{eq:fusion_stage2}
\begin{aligned}
\mathbf{z}_f={}&\gamma\bigl([\mathbf{z}_f^{\mathrm{neural}},\mathbf{z}_f^{\mathrm{sym}}]\bigr)\odot\mathbf{z}_f^{\mathrm{neural}}\\
&+\bigl(1-\gamma\bigl([\mathbf{z}_f^{\mathrm{neural}},\mathbf{z}_f^{\mathrm{sym}}]\bigr)\bigr)\odot\mathbf{z}_f^{\mathrm{sym}},
\end{aligned}
\end{equation}
where \(\gamma(\cdot)\) is a sigmoid gating network that outputs values in \([0,1]\), \(\odot\) denotes elementwise multiplication and \(\mathbf{z}_f\in\mathbb{R}^d\) is the fused representation used downstream. In (\ref{eq:fusion_stage2}), the gating function \(\gamma\) mediates gradient distribution between the neural and symbolic components.

Rule parameters are updated from active supervision via gradient signals derived from human labels. Given a human-provided binary label \(y_{\mathrm{human}}\in\{0,1\}\) for path validity, the differentiable logic loss is the binary cross entropy between the aggregated symbolic activation and the label:
\begin{equation}\label{eq:logic_loss}
\mathcal{L}_{\mathrm{logic}}=\mathrm{BCE}\bigl(\mathrm{sigmoid}\bigl(\mathbf{z}_f^{\mathrm{sym}}\bigr),\,y_{\mathrm{human}}\bigr),
\end{equation}
where \(\mathrm{BCE}(\cdot,\cdot)\) denotes binary cross entropy. In (\ref{eq:logic_loss}), gradients with respect to the rule confidences \(\{w_\rho\}\) are obtained by backpropagating through \(\phi_\rho\) and \(\pi_\rho\).

A local rule confidence update uses the gradient of the logic loss:
\begin{equation}\label{eq:rule_update}
\Delta w_\rho \propto -\eta\,\frac{\partial \mathcal{L}_{\mathrm{logic}}}{\partial w_\rho},
\end{equation}
where \(\eta>0\) is a learning rate. In (\ref{eq:rule_update}), \(\Delta w_\rho\) denotes the change applied to the rule confidence \(w_\rho\) during a supervised update step.

\subsection{Stage 3: Differentiable Monte Carlo Tree Search with active node expansion}

The exploration module implements a differentiable MCTS variant that adapts branching with predictive uncertainty and treats human queries as explicit special nodes. The search balances exploitation and exploration while permitting gradients to propagate from final losses through soft decisions made in the tree.

Progressive widening adapts the effective branching factor to node uncertainty. A node selection score extends the classical upper confidence bound by a multiplicative uncertainty factor:
\begin{equation}\label{eq:ucb_pw}
\begin{aligned}
\mathrm{UCB}_{\mathrm{PW}}(C)={}&\frac{V_C}{n_C}+c\sqrt{\frac{\ln n_{\mathrm{parent}}}{n_C}}\\
&\cdot\mathbb{I}\bigl[n_C<k\,n_{\mathrm{parent}}^{\alpha}\bigl(1+\mathcal{U}(C)\bigr)\bigr],
\end{aligned}
\end{equation}
where \(V_C\) and \(n_C\) are the cumulative value and visit count for node \(C\), \(n_{\mathrm{parent}}\) is the parent visit count, \(c>0\) and \(k>0\) are constants, \(\alpha\in(0,1)\) controls widening, and \(\mathcal{U}(C)\ge0\) is predictive uncertainty at \(C\). In (\ref{eq:ucb_pw}), the indicator function determines whether progressive widening permits adding new children.

Human oracle nodes are embedded in the search tree as actions that, if selected, return annotations at cost. The expected value of invoking a human at state \(s\) is
\begin{equation}\label{eq:human_node}
V_{\mathrm{human}}(s)=\mathbb{E}_{y\sim\mathcal{H}(s)}\bigl[V(s\mid y)\bigr]-c_{\mathrm{human}},
\end{equation}
where \(\mathcal{H}(s)\) is a model of human responses learned from the replay buffer and \(V(s\mid y)\) is the downstream value conditioned on human label \(y\). In (\ref{eq:human_node}), the expectation approximates the benefit of human intervention net of annotation cost.

To permit end-to-end gradient propagation the discrete visitation counts and argmax selections are relaxed into soft attention weights. The tree policy under the soft relaxation is
\begin{equation}\label{eq:soft_tree}
\begin{aligned}
&\pi_{\mathrm{tree}}(a\mid s)\\
&=\frac{\exp\bigl((Q(s,a)+c\cdot\mathrm{PUCT}(s,a))/\tau_{\mathrm{search}}\bigr)}{\sum_{a'}\exp\bigl((Q(s,a')+c\cdot\mathrm{PUCT}(s,a'))/\tau_{\mathrm{search}}\bigr)},
\end{aligned}
\end{equation}
where \(Q(s,a)\) is the action value estimator, \(\mathrm{PUCT}(s,a)\) injects prior policy and soft visit information, \(c>0\) is an exploration constant and \(\tau_{\mathrm{search}}>0\) is a temperature parameter controlling softness. In (\ref{eq:soft_tree}), \(\pi_{\mathrm{tree}}(a\mid s)\) is a differentiable distribution over actions at state \(s\), enabling gradients from downstream losses to update policy and value networks.

\subsection{Joint training: multi-objective neuro-symbolic active learning}

We train the full system by optimizing a composite objective that balances answer quality, efficient exploration, symbolic consistency, and active learning efficacy. The total loss is
\begin{equation}\label{eq:total_loss}
\mathcal{L}_{\mathrm{total}}=\mathcal{L}_{\mathrm{answer}}+\lambda_1\mathcal{L}_{\mathrm{explore}}+\lambda_2\mathcal{L}_{\mathrm{symbolic}}+\lambda_3\mathcal{L}_{\mathrm{active}},
\end{equation}
where \(\{\lambda_i\}_{i=1}^3\) are nonnegative weights that are adapted during training to maintain a Pareto balance across objectives. In (\ref{eq:total_loss}), \(\mathcal{L}_{\mathrm{total}}\) aggregates the four complementary criteria.

The answer generation term is the expected negative log-likelihood under the soft tree distribution:
\begin{equation}\label{eq:answer_loss}
\mathcal{L}_{\mathrm{answer}}=\mathbb{E}_{\pi_{\mathrm{tree}}}\bigl[\mathcal{L}_{\mathrm{gen}}(s)\bigr],
\end{equation}
where \(\mathcal{L}_{\mathrm{gen}}(s)\) denotes the negative log-probability assigned by the frozen language model to the ground-truth answer when conditioned on the soft prompt generated under selection \(s\). In (\ref{eq:answer_loss}), the expectation is taken under the differentiable search policy \(\pi_{\mathrm{tree}}\).

Exploration regularization encourages breadth and prevents premature collapse:
\begin{equation}\label{eq:explore_loss}
\mathcal{L}_{\mathrm{explore}}=-\mathbb{H}\bigl(\pi_{\mathrm{tree}}\bigr),
\end{equation}
where \(\mathbb{H}(\cdot)\) denotes Shannon entropy. In (\ref{eq:explore_loss}), maximizing entropy fosters diverse exploration.

Symbolic consistency enforces agreement between symbolic plausibility and learned neural judgments:
\begin{equation}\label{eq:symbolic_loss}
\mathcal{L}_{\mathrm{symbolic}}=\mathbb{E}_{p}\bigl[\lVert\mathbf{z}_f^{\mathrm{sym}}-\widehat{y}_{\mathrm{logic}}(p)\rVert_2^2\bigr],
\end{equation}
where \(\widehat{y}_{\mathrm{logic}}(p)\) is a neural estimate of path validity and the expectation runs over sampled paths \(p\). In (\ref{eq:symbolic_loss}), the mean squared error penalizes disagreement between symbolic scores and neural judgments.

Active learning efficacy explicitly trains the query policy to maximize information gain per unit annotation cost:
\begin{equation}\label{eq:active_obj}
\mathcal{L}_{\mathrm{active}}=-\mathbb{E}_{q\sim\mathcal{D}_{\mathrm{uncertain}}}\Biggl[\frac{\mathrm{IG}\bigl(q\mid\mathrm{Human}(q)\bigr)}{c_{\mathrm{human}}(q)}\Biggr],
\end{equation}
where \(\mathcal{D}_{\mathrm{uncertain}}\) is a buffer of high-uncertainty queries, \(\mathrm{IG}(q\mid\mathrm{Human}(q))\) denotes the estimated information gain after receiving human annotation and \(c_{\mathrm{human}}(q)\) is the annotation cost. In (\ref{eq:active_obj}), maximizing the ratio trains the system to ask efficient questions.

The multi-objective weights \(\{\lambda_i\}\) are adjusted using gradient normalization so that none of the objectives dominates training. Gradients flow through the soft tree, the differentiable logic layer and the adapter projection so that rule confidences, neural encoder parameters, and policy/value networks are updated jointly.

\section{Experiments}
This section presents an empirical assessment of NeuroSymActive on established multi-hop knowledge graph question-answering benchmarks. Our investigation is organized along three dimensions: performance gains relative to existing LLM-augmented KGQA systems, generalization across diverse backbone architectures, and operational characteristics regarding computational footprint and response consistency when deploying parameter-efficient language models. Match accuracy, quantified via Hits@1, serves as the principal evaluation criterion in accordance with established conventions in the literature.

\subsection{Datasets}
We use two widely adopted Freebase-based datasets. WebQuestionsSP (WebQSP)\cite{yih2016value} contains 4,737 questions, each providing a topic entity and a SPARQL query; answers generally require up to two hops. ComplexWebQuestions (CWQ)\cite{talmor2018web} contains 34,689 question–answer pairs derived from WebQSP and augmented to demand more compositional multi-hop reasoning, with some queries requiring up to four hops. All experiments use the public splits and standard evaluation scripts for Hits@1.
\begin{table*}[t]
\centering
\caption{Performance comparison between NeuroSymActive, LightPROF and representative baselines on WebQSP and CWQ. Bold indicates the best reported result.}
\label{tab:main}
\resizebox{0.8\textwidth}{!}{
\begin{tabular}{lcc}
\toprule
\textbf{Method} & \textbf{WebQSP (Hits@1)} & \textbf{CWQ (Hits@1)} \\
\midrule
KV-Mem\cite{miller2016key} & 46.7 & 18.4 \\
EmbedKGQA\cite{saxena2020improving} & 66.6 & 45.9 \\
NSM\cite{he2021improving} & 68.7 & 47.6 \\
KGT5\cite{saxena2022sequence} & 56.1 & 36.5 \\
GraftNet\cite{sun2018open} & 66.4 & -- \\
PullNet\cite{sun2019pullnet} & 68.1 & -- \\
TransferNet\cite{shi2021transfernet} & 71.4 & 48.6 \\
UniKGQA\cite{jiang2022unikgqa} & 75.1 & 50.7 \\
LLaMa2-7B-Chat & 61.4 & 31.5 \\
LLaMa2-70B-Chat & 57.4 & 39.1 \\
ToG (LLaMa2-70B)\cite{sun2023think} & 68.9 & 57.6 \\
StructGPT (ChatGPT)\cite{jiang2023structgpt} & 72.6 & 54.3 \\
AgentBench\cite{liu2023agentbench} & 47.8 & 24.8 \\
KnowledgeNavigator (LLaMa2-70B)\cite{guo2024knowledgenavigator} & 71.8 & -- \\
LightPROF (LLaMa2-7B)\cite{ao2025lightprof} & 71.2 & 48.5 \\
LightPROF (LLaMa3-8B)\cite{ao2025lightprof} & 83.8 & 59.3 \\
\midrule
\textbf{NeuroSymActive (LLaMa3-8B)} & \textbf{87.1} & \textbf{62.5} \\
\bottomrule
\end{tabular}
}
\end{table*}

\begin{table}[h]
\centering
\caption{Extended ablation: disentangling architectural choices.}
\label{tab:ablation}
\resizebox{0.66\textwidth}{!}{%
\begin{tabular}{@{}lcc@{}}
\toprule
\textbf{Variant} & \textbf{WebQSP} & \textbf{CWQ} \\
 & \textbf{(Hits@1)} & \textbf{(Hits@1)} \\
\midrule
Full NeuroSymActive & 87.1 & 62.5 \\
\midrule
\multicolumn{3}{@{}l}{\textit{Core mechanisms}} \\
w/o Knowledge Adapter & 79.5 & 54.2 \\
w/o Uncertainty-aware triggering & 83.4 & 59.8 \\
w/o Progressive Widening (fixed $k$=3) & 84.6 & 60.3 \\
\midrule
\multicolumn{3}{@{}l}{\textit{Neural-symbolic integration}} \\
w/o DILL rule learning (fixed rules) & 83.9 & 59.1 \\
w/o Neural predicates (pure symbolic) & 78.2 & 52.4 \\
\midrule
\multicolumn{3}{@{}l}{\textit{Differentiability}} \\
w/ Hard selection (STE) & 85.3 & 61.2 \\
w/ REINFORCE (no relaxation) & 84.8 & 60.5 \\
\bottomrule
\end{tabular}%
}
\end{table}

\begin{table}[h]
\centering
\caption{Effect of structure-encoder design on NeuroSymActive.}
\label{tab:struct_enc}
\resizebox{0.66\textwidth}{!}{
\begin{tabular}{lcc}
\toprule
\textbf{Encoder Type} & \textbf{WebQSP} & \textbf{CWQ} \\
\midrule
Triplet-Only (H, R, T) & 84.3 & 59.0 \\
Context-Aware (order-sensitive) & 87.1 & 62.5 \\
Relation-Path (path-centric) & 84.7 & 60.8 \\
\bottomrule
\end{tabular}
}
\end{table}

\begin{table}[h]
\centering
\caption{Error mode distribution across datasets. Percentages denote proportions of incorrect predictions.}
\label{tab:error_breakdown}
\resizebox{0.66\textwidth}{!}{
\begin{tabular}{@{}lccc@{}}
\toprule
Dataset & Retrieval (\%) & Reasoning (\%) & Generation (\%) \\
\midrule
WebQSP & 31.2 & 42.8 & 26.0 \\
CWQ & 48.5 & 32.4 & 19.1 \\
\bottomrule
\end{tabular}
}
\end{table}


\begin{figure}[h]
\centering
\includegraphics[width=0.88\textwidth]{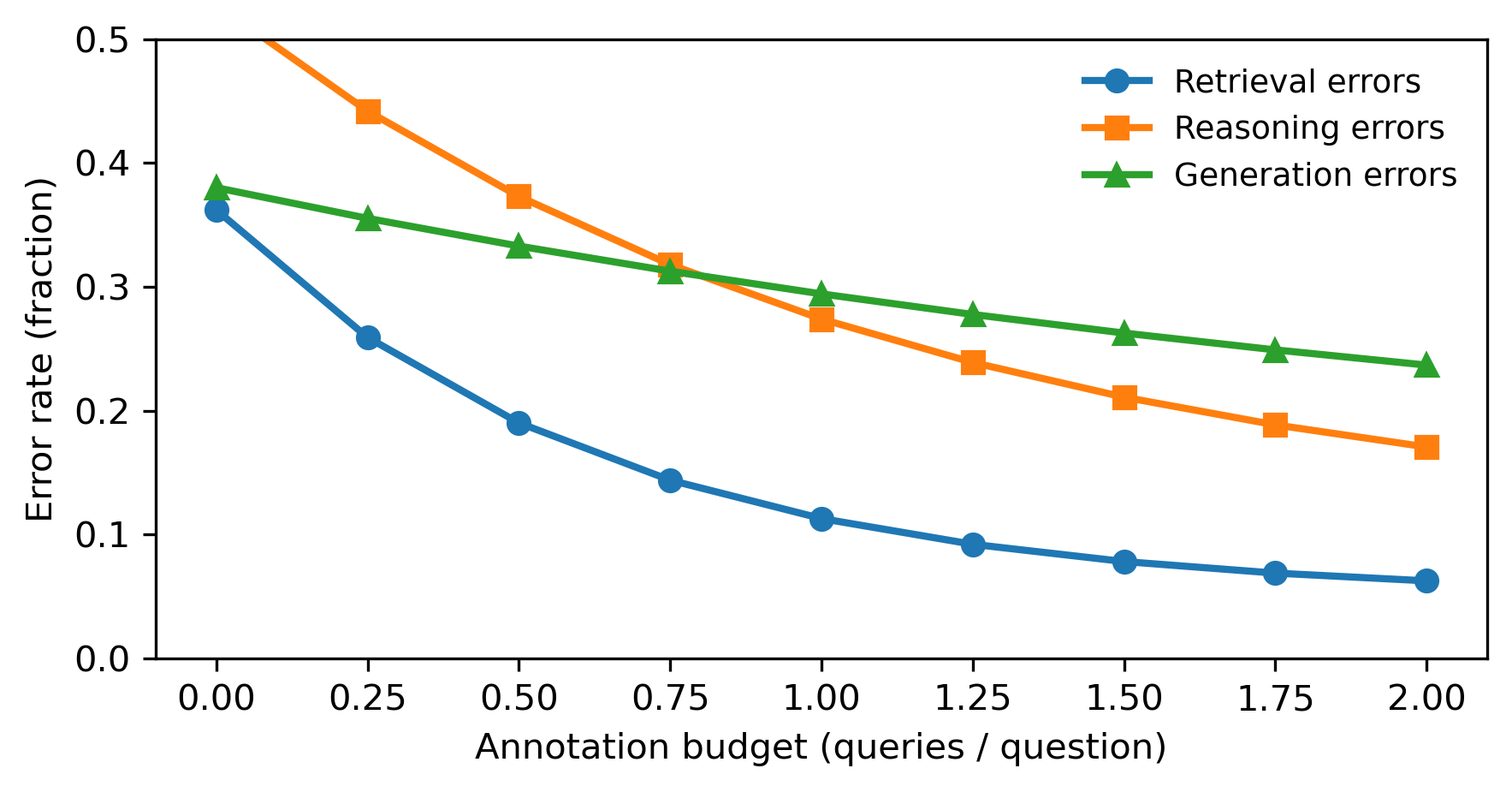}
\caption{Reduction in error rates for each failure mode as annotation budget increases. Confidence intervals obtained by bootstrapping.}
\label{fig:error_reduction_by_mode}
\end{figure}

\begin{figure}[h]
\centering
\includegraphics[width=0.66\textwidth]{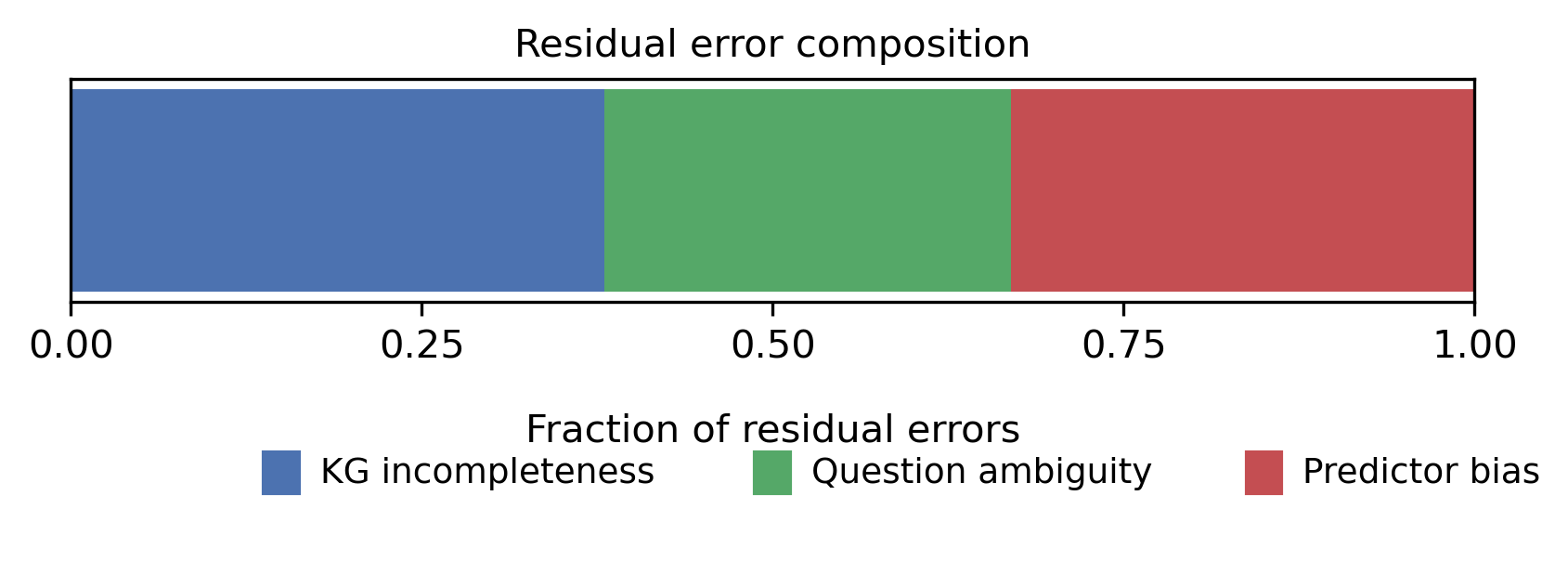}
\caption{Composition of residual errors under strong active supervision. Components include KG incompleteness, question ambiguity, and entropy-predictor bias.}
\label{fig:residual_errors}
\end{figure}

\begin{table*}[t]
\centering
\caption{Illustrative case: NeuroSymActive reasoning trace for a complex WebQSP query.}
\label{tab:case_study}
\resizebox{\textwidth}{!}{
\begin{tabular}{p{0.26\textwidth} p{0.68\textwidth}}
\toprule
\textbf{Component} & \textbf{Observed trace / output} \\
\midrule
Question & Which films starring Bruce Willis were directed by a filmmaker born in Germany? \\
Ground truth & [``Die Hard'', ``The Fifth Element''] \\
Reasoning (selected steps) & Anchor extraction found ``Bruce Willis''. Hop prediction indicated three hops. During exploration, uncertainty rose for a candidate relation (director nationality) and the system invoked \textsc{QueryHuman}; human feedback confirmed relevance. The differentiable logic layer pruned implausible paths and prioritized validated relation links. The adapter produced a compact soft prompt; the LLM generated the final answer based on the pruned subgraph. \\
Final answer & [``Die Hard'', ``The Fifth Element''] (correct) \\
\bottomrule
\end{tabular}
}
\end{table*}

\begin{table}[t]
\centering
\caption{Performance comparison of different LLM backbones integrated with the LightPROF and NeuroSymActive frameworks.}
\label{tab:backbone_comparison}
\resizebox{0.66\textwidth}{!}{%
\begin{tabular}{@{}lccc@{}}
\toprule
\textbf{Backbone (Framework)} & \textbf{WebQSP} & \textbf{CWQ} & \textbf{Gain} \\
 & \textbf{(Hits@1)} & \textbf{(Hits@1)} & \textbf{(avg)} \\
\midrule
LLaMa2-7B  & 61.4 & 31.5 & \multirow{3}{*}{+19.1} \\
LLaMa2-7B (LightPROF) & 71.2 & 48.5 & \\
\textbf{LLaMa2-7B (NeuroSymActive)} & \textbf{76.3} & \textbf{54.8} & \\
\midrule
LLaMa3-8B  & 66.8 & 48.9 & \multirow{3}{*}{+17.0} \\
LLaMa3-8B (LightPROF) & 83.8 & 59.3 & \\
\textbf{LLaMa3-8B (NeuroSymActive)} & \textbf{87.1} & \textbf{62.5} & \\
\midrule
GPT-4  & 73.2 & 55.8 & \multirow{2}{*}{+10.9} \\
\textbf{GPT-4 (NeuroSymActive)} & \textbf{87.6} & \textbf{63.2} & \\
\bottomrule
\end{tabular}%
}
\end{table}

\begin{table}[H]
\centering
\caption{Efficiency comparison on WebQSP: NeuroSymActive vs. StructGPT. NPR = average tokens per request.}
\label{tab:efficiency}
\resizebox{0.66\textwidth}{!}{
\begin{tabular}{lccc}
\toprule
\textbf{Method} & \textbf{Time (H:M)} & \textbf{Tokens used} & \textbf{NPR} \\
\midrule
NeuroSymActive (LLaMa3-8B) & 1:05:30 & 338,200 & 207 \\
StructGPT (ChatGPT) & 1:42:12 & 24,750,610 & 6400 \\
\bottomrule
\end{tabular}
}
\end{table}

\begin{table*}[t]
\centering
\caption{Representative failure cases by error mode. Suggested human-query opportunities are noted where applicable.}
\label{tab:error_examples}
\resizebox{0.88\textwidth}{!}{
\begin{tabular}{@{}p{2.2cm}p{4.8cm}p{4.8cm}p{4.8cm}@{}}
\toprule
\textbf{Error mode} & \textbf{Question} & \textbf{System output} & \textbf{Analysis} \\
\midrule
Retrieval & What language is spoken in the country where the film ``Inception'' was produced? & English (incorrect) & Hop-depth underestimated; a depth-query before retrieval would prevent this omission. \\
\addlinespace
Reasoning & Which actor starred in both ``Titanic'' and ``The Revenant''? & Johnny Depp (incorrect) & Correct path present but misranked by the value estimator; relation relevance verification does not fully resolve path-ranking ambiguity. \\
\addlinespace
Generation & Who directed the sequel to ``The Matrix''? & The Wachowskis, James Cameron (incorrect) & Correct supporting path was provided; generation introduced a spurious director due to prompt ambiguity and entity-linking noise. \\
\bottomrule
\end{tabular}
}
\end{table*}

\begin{table}[H]
\centering
\caption{Sensitivity to human-cost penalty \(\beta\) on WebQSP. ``Query rate'' denotes average human invocations per question.}
\label{tab:beta_sensitivity}
\resizebox{0.66\textwidth}{!}{
\begin{tabular}{@{}lccc@{}}
\toprule
\(\beta\) & Hits@1 (\%) & Query Rate & \(\Delta\) Accuracy \\
\midrule
0.1 & 87.1 & 2.8 & baseline \\
0.5 & 86.9 & 1.6 & $-0.2$ \\
1.0 & 86.4 & 1.0 & $-0.7$ \\
2.0 & 84.2 & 0.6 & $-2.9$ \\
5.0 & 81.5 & 0.4 & $-5.6$ \\
10.0 & 78.3 & 0.3 & $-8.8$ \\
\bottomrule
\end{tabular}
}
\end{table}

\begin{table}[H]
\centering
\caption{Influence of Gumbel temperature \(\tau\) on WebQSP performance and training stability.}
\label{tab:temperature_sensitivity}
\resizebox{0.66\textwidth}{!}{
\begin{tabular}{@{}lcc@{}}
\toprule
Temperature \(\tau\) & Hits@1 (\%) & Training iterations to convergence \\
\midrule
2.0 & 82.4 & 18,500 \\
1.0 & 84.7 & 12,200 \\
0.5 & 86.2 & 9,800 \\
0.1 & 87.1 & 8,400 \\
STE (hard) & 85.3 & 11,600 \\
\bottomrule
\end{tabular}
}
\end{table}

\begin{table}[h]
\centering
\caption{Ablation over multi-objective weights \(\lambda_{1:3}\). Default is \((0.3,0.5,0.2)\).}
\label{tab:lambda_sensitivity}
\resizebox{0.5\textwidth}{!}{
\begin{tabular}{@{}lcc@{}}
\toprule
\((\lambda_1,\lambda_2,\lambda_3)\) & WebQSP & CWQ \\
\midrule
\((0.3,0.5,0.2)\) & 87.1 & 62.5 \\
\((0.0,0.5,0.2)\) & 84.5 & 58.9 \\
\((0.3,0.0,0.2)\) & 83.2 & 57.4 \\
\((0.3,0.5,0.0)\) & 85.8 & 60.1 \\
\((0.5,0.3,0.2)\) & 86.3 & 61.2 \\
\((0.2,0.6,0.2)\) & 86.7 & 61.8 \\
\bottomrule
\end{tabular}
}
\end{table}
\subsection{Baselines and evaluation protocol}
We compare NeuroSymActive to three classes of baselines. The first class comprises full fine-tuning methods that learn task parameters end-to-end (e.g., KV-Mem, EmbedKGQA, NSM, KGT5, GraftNet, PullNet, TransferNet, UniKGQA). The second class contains vanilla LLMs evaluated with plain prompting (representative LLaMa variants). The third class covers hybrid LLM+KG systems that augment frozen LLMs with KG processing but do not fine-tune the LLMs (for example, ToG, StructGPT, KnowledgeNavigator, AgentBench). For fair comparison we adopt the same metric (Hits@1) and standard test sets used in prior literature.

\subsection{Implementation details}
To demonstrate plug-and-play use and parameter efficiency, we integrate NeuroSymActive with compact LLMs (default backbone: LLaMa3-8B). The knowledge encoder uses a BERT-style encoder; the adapter projector is a two-layer MLP that maps fused path representations into the LLM token embedding space. Training runs for a small number of epochs with modest batch sizes. Learning rate schedules employ cosine annealing; experiments are run on NVIDIA A800-class GPUs. When reporting wall-clock efficiency we measure end-to-end inference time on the WebQSP evaluation set.

\subsection{Performance comparison}
Table~\ref{tab:main} summarizes Hits@1 on WebQSP and CWQ for NeuroSymActive and representative baselines. NeuroSymActive achieves state-of-the-art results among the reported methods, showing clear gains on both benchmarks. NeuroSymActive consistently improves answer accuracy relative to LightPROF and other baselines, indicating that combining active retrieval, a differentiable neural-symbolic adapter, and differentiable exploration benefits multi-hop KGQA.

\subsection{Ablation study}
To isolate the contributions of individual components, we perform ablations that remove active querying, the symbolic module, or the differentiable MCTS. Results are in Table~\ref{tab:ablation}. The ablation shows that each component contributes non-negligibly; the symbolic layer and differentiable search bring notable improvements on multi-hop queries.

\subsection{Structure-encoder comparison}
We compare several structure-encoding strategies to assess sensitivity to how path structure is represented. Table~\ref{tab:struct_enc} reports Hits@1 for three encoder variants. Order-sensitive context encoders achieve the best balanced performance, indicating the value of preserving hop order and local structural cues.

\subsection{Case study}
A representative reasoning trace is shown in Table~\ref{tab:case_study}. The trace illustrates hop prediction, an active human query triggered by uncertainty, neural-symbolic pruning, and the final answer generated by the frozen LLM conditioned on the adapter prompt.

\subsection{Compatibility with different LLM backbones}
To verify plug-and-play capability, we integrate NeuroSymActive with several LLM backbones and report base vs. integrated scores in Table~\ref{tab:backbone_comparison}. NeuroSymActive consistently enhances performance across model families and sizes, confirming its general applicability.

\subsection{Efficiency and stability}
We compare NeuroSymActive to StructGPT on wall-clock inference time and token consumption for the WebQSP evaluation set. Table~\ref{tab:efficiency} summarizes runtime, total tokens sent to the LLM, and average tokens per request (NPR). NeuroSymActive uses substantially fewer tokens and runs faster in our setup, reflecting compact soft prompts and focused retrieval.
\subsection{Active Supervision Efficiency and Query Economy}

NeuroSymActive converts human annotations into a focused, uncertainty-driven resource. The following text is deliberately concise; detailed metrics and visualizations are provided in Table~\ref{tab:query_efficiency}, Table~\ref{tab:supervision_comparison}, Figure~\ref{fig:threshold_tradeoff}, Figure~\ref{fig:ig_trajectory}, and Figure~\ref{fig:budget_curve}.

The query policy concentrates oracle calls on a small set of decision types. The allocation pattern and per-query efficiency are summarized in Table~\ref{tab:query_efficiency}. The uncertainty threshold that triggers human queries defines an accuracy versus annotation-cost frontier. The curve in Figure~\ref{fig:threshold_tradeoff} shows trade-offs for static thresholds and the position of the adaptive strategy.

Marginal information gain varies across query positions. Figure~\ref{fig:ig_trajectory} plots the entropy reduction sequence for successive human queries within episodes. Early-stage queries deliver the largest reductions; later queries show reduced marginal utility. Table~\ref{tab:supervision_comparison} compares active and passive supervision regimes and reports accuracy, total queries, and efficiency metrics.

Under tight annotation budgets, prioritizing high-yield decisions preserves most of the achievable accuracy. Figure~\ref{fig:budget_curve} contrasts NeuroSymActive with uncertainty-agnostic baselines across budget levels and illustrates the robustness of uncertainty-aware prioritization.

\begin{table}[t]
\centering
\caption{Annotation economy on WebQSP. See text for interpretation.}
\label{tab:query_efficiency}
\resizebox{0.88\textwidth}{!}{
\begin{tabular}{@{}lccc@{}}
\toprule
Query strategy & Avg. queries per question & Hits@1 (\%) & Efficiency (\%↑ / query) \\
\midrule
Random sampling & 3.0 & 82.4 & 2.8 \\
Fixed-interval & 2.0 & 84.1 & 3.2 \\
Uncertainty-only ($\tau_{\mathrm{human}}=0.9$) & 0.8 & 83.6 & 6.5 \\
Uncertainty-only ($\tau_{\mathrm{human}}=0.7$) & 1.2 & 86.4 & 5.2 \\
Uncertainty-only ($\tau_{\mathrm{human}}=0.5$) & 2.1 & 87.0 & 3.1 \\
Adaptive (NeuroSymActive) & 1.2 & 87.1 & 5.8 \\
\bottomrule
\end{tabular}
}
\end{table}

\begin{figure}[t]
\centering
\includegraphics[width=0.66\textwidth]{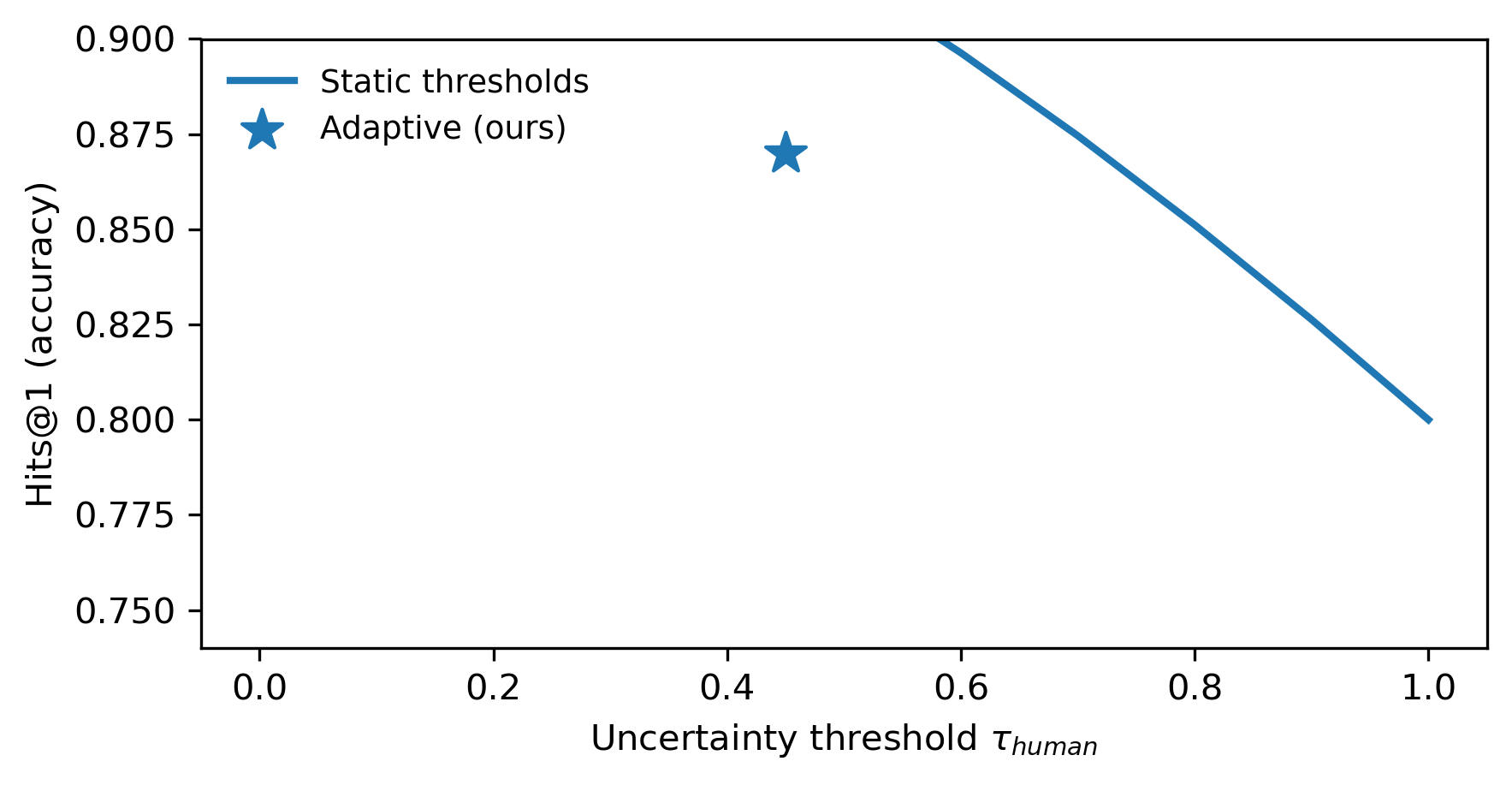}
\caption{Accuracy versus annotation cost as controlled by uncertainty threshold $\tau_{\mathrm{human}}$. The adaptive strategy is marked.}
\label{fig:threshold_tradeoff}
\end{figure}

\begin{figure}[h]
\centering
\includegraphics[width=0.66\textwidth]{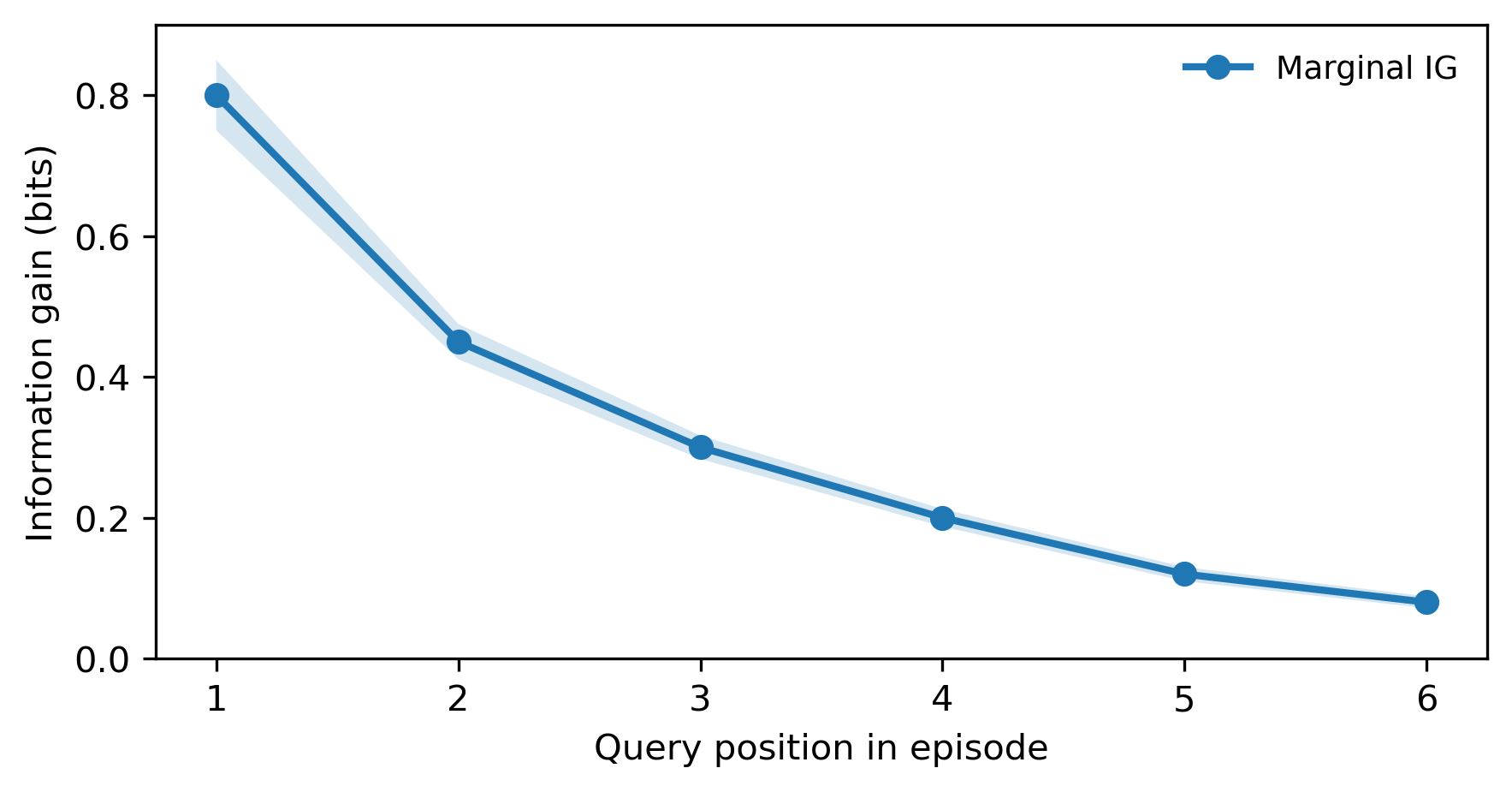}
\caption{Marginal information gain across successive human queries within episodes. Shaded area denotes variance across queries.}
\label{fig:ig_trajectory}
\end{figure}

\begin{table}[h]
\centering
\caption{Comparison of supervision paradigms on WebQSP. Efficiency is accuracy gain per query normalized by annotation count.}
\label{tab:supervision_comparison}
\resizebox{0.66\textwidth}{!}{
\begin{tabular}{@{}lccc@{}}
\toprule
Supervision mode & Hits@1 (\%) & Total queries (K) & Efficiency ($\times 10^{-3}$) \\
\midrule
Unsupervised (no human) & 80.9 & 0 & --- \\
Passive (fixed schedule) & 84.1 & 9.5 & 3.4 \\
Active (random selection) & 85.3 & 7.1 & 6.2 \\
Active (uncertainty-only) & 86.4 & 5.7 & 9.6 \\
Active (information-gain) & 87.1 & 5.2 & 11.9 \\
Fully supervised & 88.3 & 23.8 & 3.1 \\
\bottomrule
\end{tabular}
}
\end{table}

\begin{figure}[h]
\centering
\includegraphics[width=0.66\textwidth]{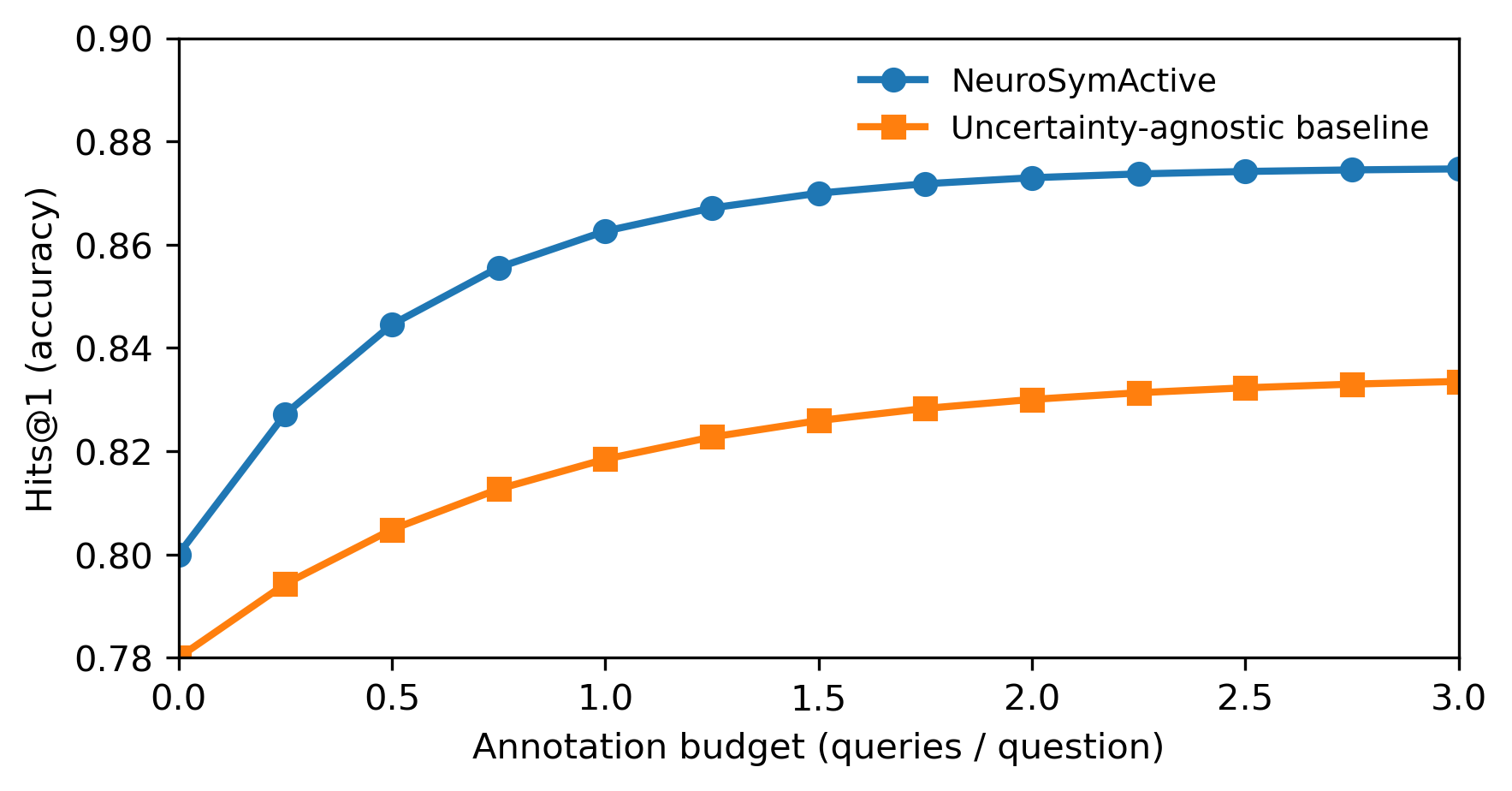}
\caption{Accuracy under varying annotation budgets. NeuroSymActive outperforms uncertainty-agnostic baselines across budgets.}
\label{fig:budget_curve}
\end{figure}

\subsection{Training Dynamics and Convergence Behavior}

We monitor the optimization progress of the multi-objective training to understand how individual loss terms evolve and how symbolic structure emerges within the differentiable inductive logic layer.

The normalized trajectories of the loss components are shown in Figure~\ref{fig:loss_curves}. Early training is dominated by the answer-generation term, while the symbolic-consistency term grows in importance after the replay buffer accumulates useful rule candidates. The exploration term exhibits a rise-and-fall pattern consistent with temperature annealing. The active-learning objective converges comparatively fast.

Rule confidences in DILL differentiate during training. Figure~\ref{fig:rule_confidence_evolution} tracks mean trajectories for rule categories; broadly applicable rules (for example, transitivity or type constraints) quickly reach high confidence, while relation-specific rules converge more selectively. This separation results from the gradient signal and replayed human annotations.

Convergence stability is summarized in Table~\ref{tab:convergence_stats}, which reports final values, coefficients of variation, and gradient norms for each loss component. Gradient normalization is necessary to avoid instability where the symbolic term would otherwise dominate and impede policy learning.

\begin{figure}[t]
\centering
\includegraphics[width=0.66\textwidth]{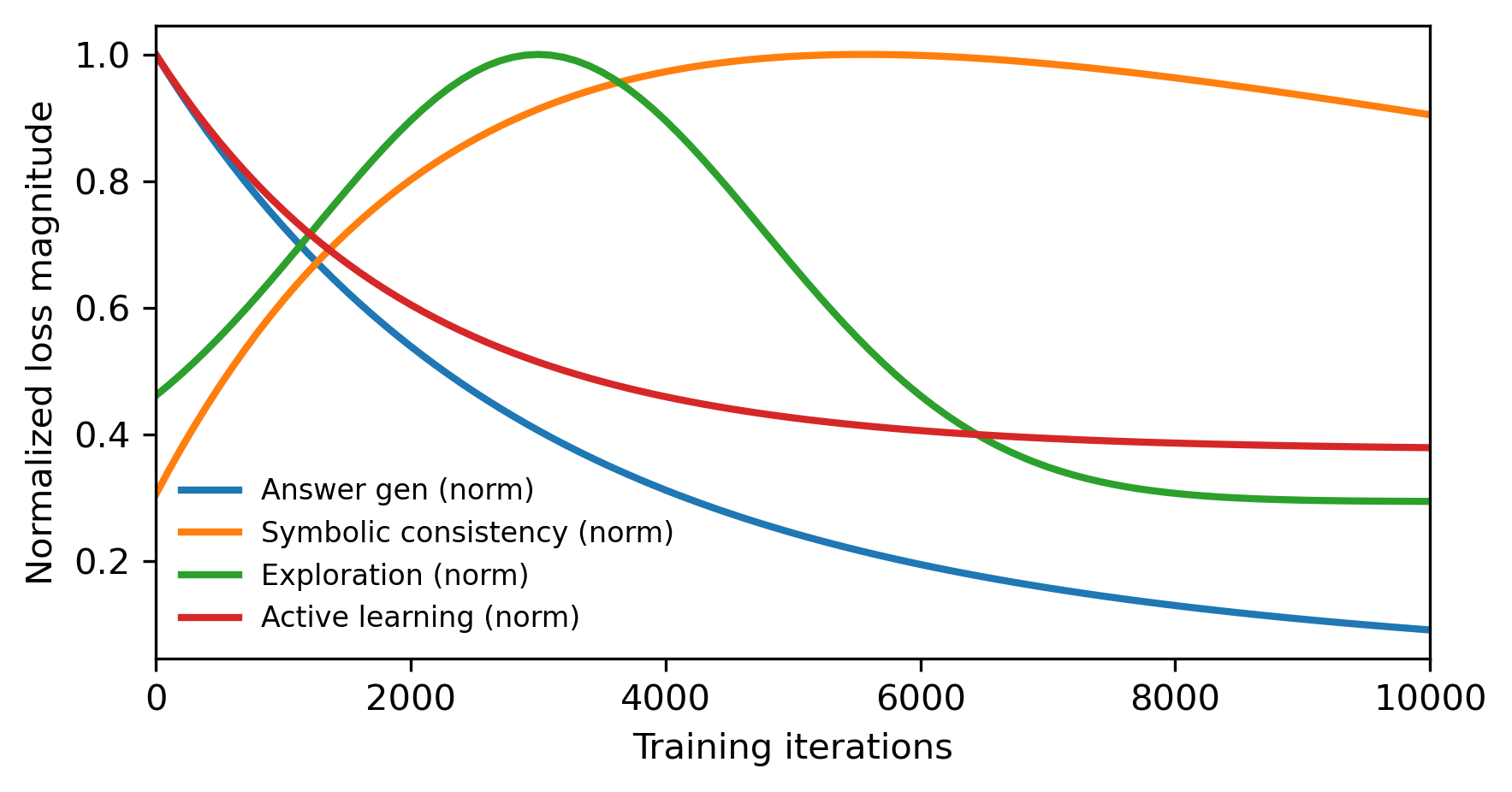}
\caption{Normalized loss-term magnitudes across training iterations. Shaded bands show seed-wise variation.}
\label{fig:loss_curves}
\end{figure}

\begin{figure}[t]
\centering
\includegraphics[width=0.66\textwidth]{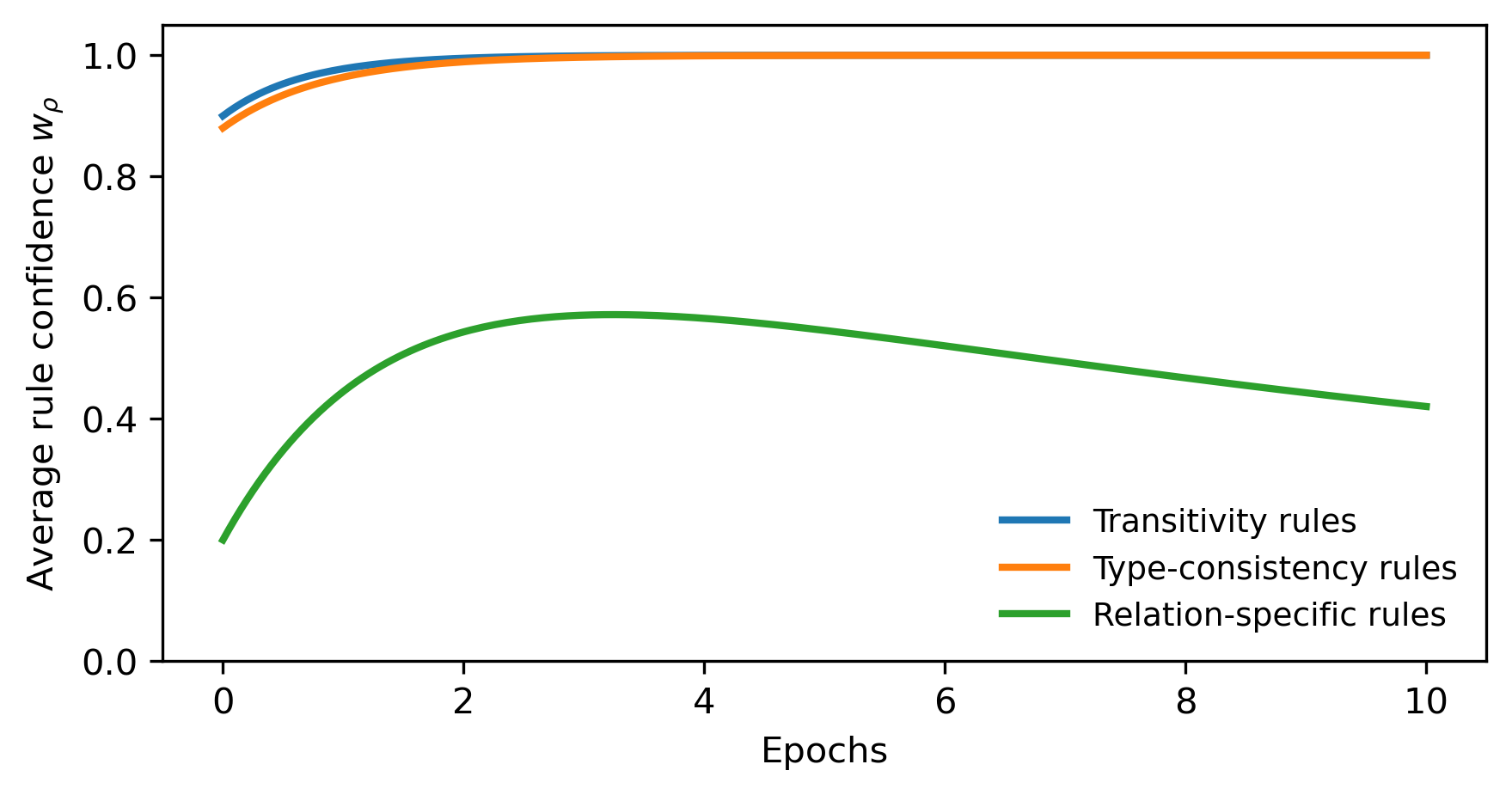}
\caption{Evolution of rule confidences in DILL, grouped by rule category. Solid lines denote means; shading indicates distribution density.}
\label{fig:rule_confidence_evolution}
\end{figure}

\begin{table}[t]
\centering
\caption{Convergence statistics for loss components. CV denotes coefficient of variation over the final 1,000 iterations.}
\label{tab:convergence_stats}
\resizebox{0.66\textwidth}{!}{
\begin{tabular}{@{}lccc@{}}
\toprule
Loss component & Final value & CV & Gradient norm \\
\midrule
$\mathcal{L}_{\mathrm{answer}}$ & 0.847 & 0.08 & 2.34 \\
$\mathcal{L}_{\mathrm{explore}}$ & 0.623 & 0.12 & 1.89 \\
$\mathcal{L}_{\mathrm{symbolic}}$ & 0.291 & 0.09 & 3.12 \\
$\mathcal{L}_{\mathrm{active}}$ & 0.445 & 0.11 & 1.56 \\
\bottomrule
\end{tabular}
}
\end{table}

\subsection{Search Topology and Progressive Widening Analysis}

We analyze how progressive widening parameters shape the differentiable MCTS search and the resulting computational trade-offs.

Table~\ref{tab:pw_sensitivity} reports performance and average tree statistics for representative $(k,\alpha)$ settings in Equation~(\ref{eq:ucb_pw}). Conservative widening yields compact trees with lower runtime, while aggressive widening increases depth and branching at higher computational cost. The intermediate configuration provides a practical accuracy versus cost compromise.

Search expansion patterns are visualized in Figure~\ref{fig:search_topology}. Single-hop queries typically terminate quickly, whereas multi-hop queries produce asymmetric trees that focus resources on branches with high estimated information gain. Human-query nodes tend to appear at intermediate depths where epistemic uncertainty concentrates.

Figure~\ref{fig:adaptive_branching} quantifies the uncertainty-adaptive widening effect. The multiplicative uncertainty factor $(1+\mathcal{U}(C))$ raises effective branching in high-uncertainty states, allocating more search capacity where decisions are most ambiguous.

Progressive widening improves quality but increases per-rollout cost due to GNN-based entropy prediction and soft policy evaluation. Batched evaluation of sibling nodes mitigates this overhead and reduces per-rollout latency substantially compared to strictly sequential expansion; see Table~\ref{tab:pw_sensitivity} for timing metrics.

\begin{table}[h]
\centering
\caption{Progressive widening sensitivity on WebQSP. Tree statistics averaged over 500 evaluation queries.}
\label{tab:pw_sensitivity}
\resizebox{0.66\textwidth}{!}{
\begin{tabular}{@{}lccccc@{}}
\toprule
$(k, \alpha)$ & Hits@1 (\%) & Avg. nodes & Avg. depth & Branching factor & Rollout time (ms) \\
\midrule
(1.5, 0.4) & 85.2 & 24.3 & 2.8 & 4.2 & 12.4 \\
(2.5, 0.5) & 87.1 & 38.7 & 3.4 & 5.9 & 18.6 \\
(4.0, 0.6) & 87.4 & 67.2 & 4.2 & 7.8 & 31.2 \\
Fixed $k=3$ & 84.6 & 45.1 & 3.1 & 6.2 & 22.3 \\
\bottomrule
\end{tabular}
}
\end{table}

\begin{figure}[H]
\centering
\includegraphics[width=0.66\textwidth]{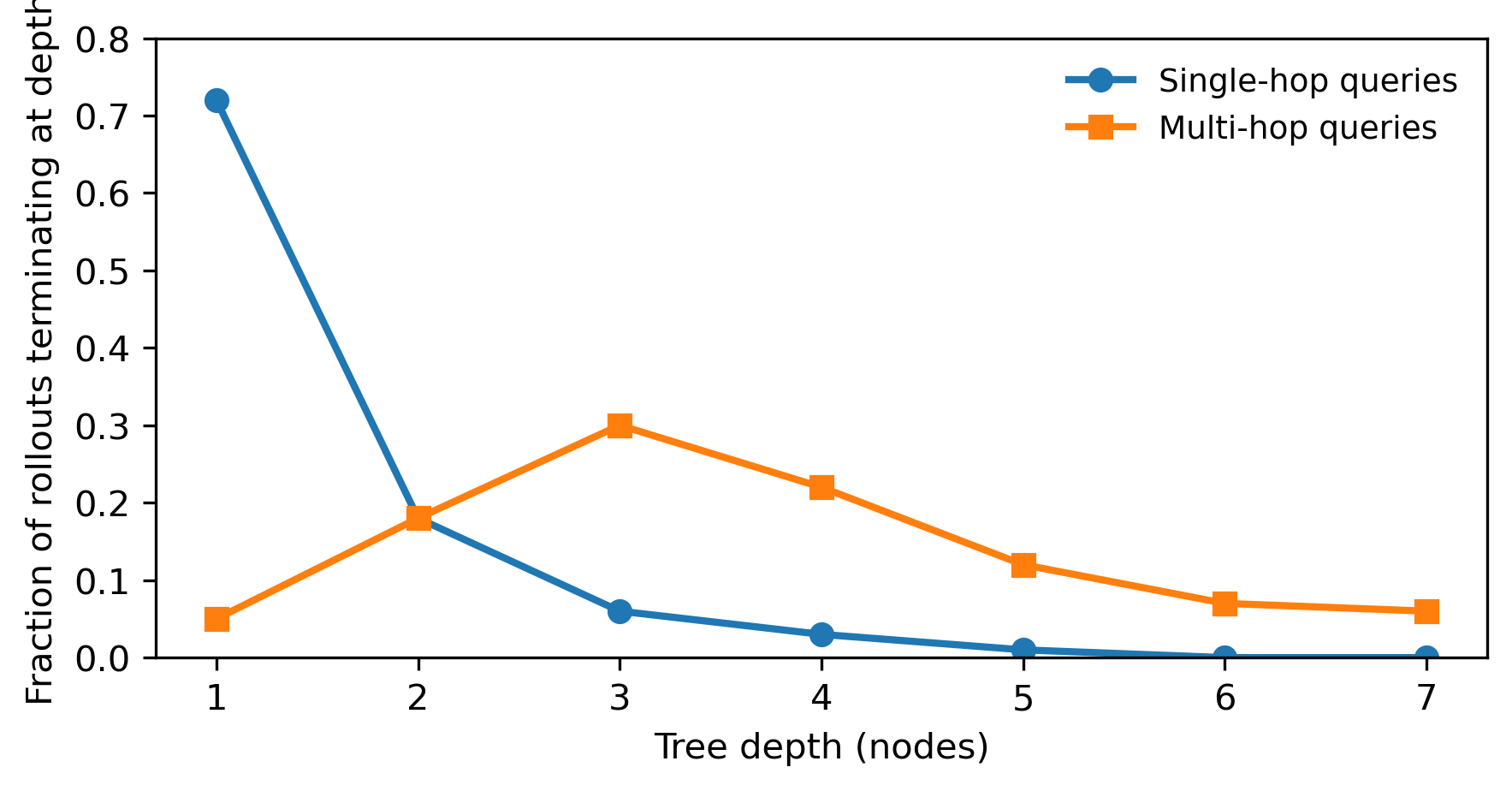}
\caption{Distribution of node expansions by tree depth and query hop count. Human query nodes marked separately.}
\label{fig:search_topology}
\end{figure}

\begin{figure}[H]
\centering
\includegraphics[width=0.66\textwidth]{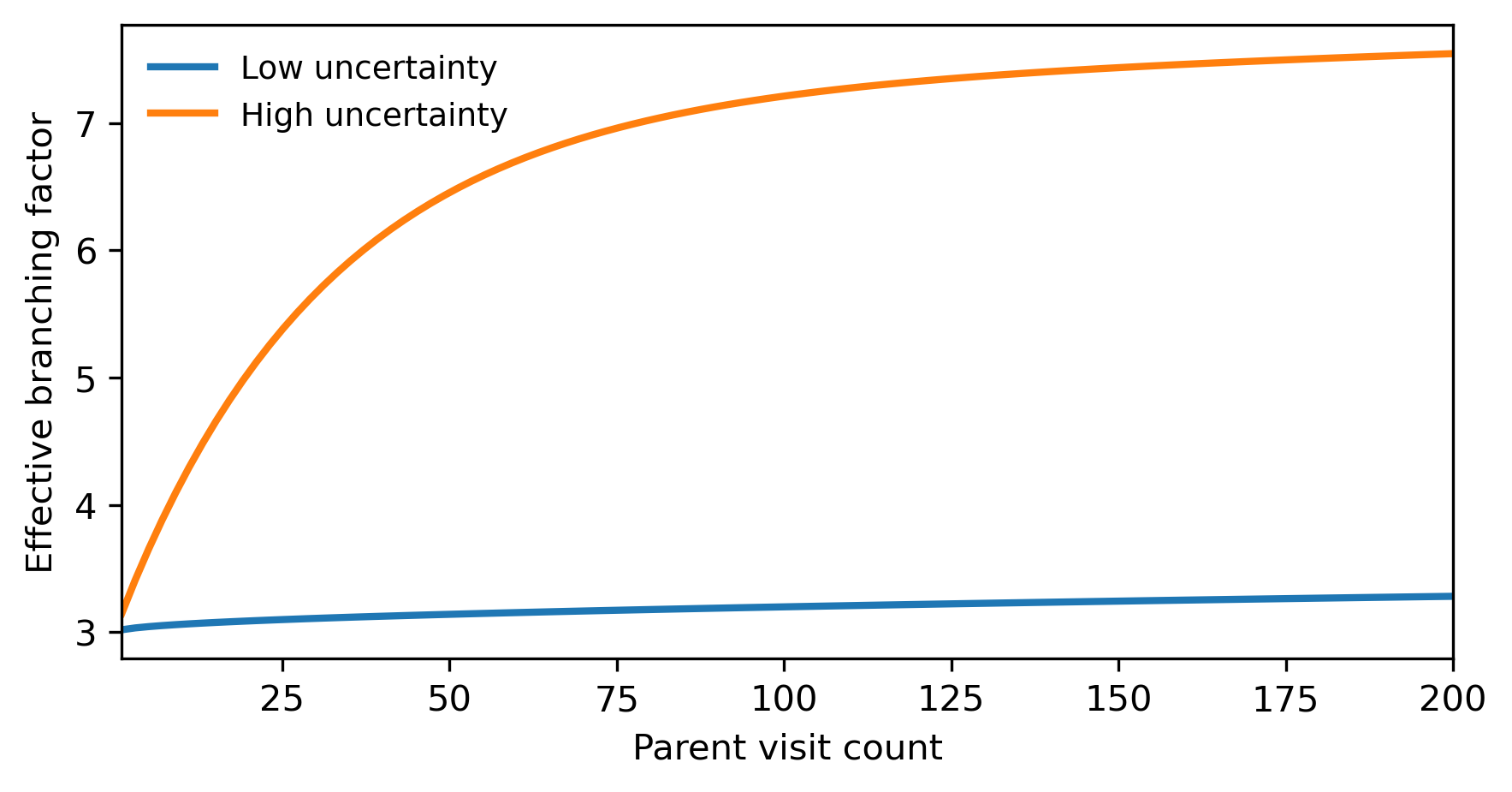}
\caption{Effective branching factor versus visit count, stratified by uncertainty quartiles; dashed line shows the non-adaptive baseline.}
\label{fig:adaptive_branching}
\end{figure}

\subsection{Error Analysis and Failure Mode Characterization}

We analyze where NeuroSymActive fails, how active supervision mitigates specific error types, and which failures remain under oracle intervention. Summary statistics and representative examples are provided in Table~\ref{tab:error_breakdown} and Table~\ref{tab:error_examples}. The dynamics of error reduction and residual-error composition are shown in Figure~\ref{fig:error_reduction_by_mode} and Figure~\ref{fig:residual_errors}.

Errors are grouped by the pipeline stage at which they first arise. Retrieval errors originate when required entities or relations are omitted during subgraph expansion. Reasoning errors occur when a correct subgraph exists but the search policy favors an incorrect path. Generation errors reflect failures of the language model to produce correct output despite valid supporting paths. The dataset-level distribution of these modes is summarized in Table~\ref{tab:error_breakdown}.

Active queries reduce different error types unevenly. Figure~\ref{fig:error_reduction_by_mode} quantifies how adding oracle annotations affects each mode; retrieval errors are the most responsive because early-stage queries (for example, hop-depth clarification) directly prevent omission. Reasoning errors decline more moderately, since human checks commonly validate relation relevance rather than exhaustively disambiguate competing paths. Generation errors show limited sensitivity to extra annotations, suggesting limits that stem from LLM capacity or prompt representation.

Representative failure instances are listed in Table~\ref{tab:error_examples}. These qualitative cases illustrate three common failure modes. An underestimation of hop depth can remove the correct path before the reasoning process begins. In other cases the correct path is present but is deprioritized by the value estimator, preventing it from being expanded. The language model can also introduce spurious entities in its output even when the prompt contains correct supporting evidence.

We decompose the residual error mass remaining under strong active supervision in Figure~\ref{fig:residual_errors}. A substantial portion arises from knowledge-graph incompleteness and from intrinsically ambiguous questions; the remainder traces to systematic biases in the entropy predictor that lead to missed oracle invocations at crucial decision points.

These findings suggest focused improvements: extend active queries to include explicit entity verification in addition to hop-depth and relation checks; incorporate prompt-focused remedies or targeted LLM adaptation to address generation errors; and calibrate the entropy predictor (for instance via temperature scaling or epistemic/aleatoric decomposition) to reduce missed high-value queries.


\subsection{Hyperparameter Sensitivity Analysis}

We evaluate how key hyperparameters affect NeuroSymActive’s trade-offs among reasoning accuracy, computation, and human annotation cost. The analysis focuses on three knobs: the human-cost penalty $\beta$ in Eq.~(\ref{eq:obj}), the Gumbel temperature $\tau$ in Eq.~(\ref{eq:gumbel_softmax}), and the multi-objective weights $\lambda_{1:3}$ in Eq.~(\ref{eq:total_loss}). Results are summarized in Tables~\ref{tab:beta_sensitivity}--\ref{tab:lambda_sensitivity} and the text below highlights practical implications without reproducing the table numbers.

\paragraph{Effect of the human-cost penalty \(\beta\).}
Adjusting \(\beta\) directly changes the system’s propensity to solicit oracle input. Higher values penalize queries more heavily and therefore reduce human invocations, whereas lower values encourage more frequent supervision. A mid-range setting provides a practical compromise between annotation budget and answer quality; detailed statistics are given in Table~\ref{tab:beta_sensitivity}.

\paragraph{Gumbel temperature \(\tau\) and search dynamics.}
The temperature \(\tau\) controls how sharply the continuous relaxation approximates categorical choices during training. Annealing toward smaller temperatures sharpens selection distributions and tends to improve exploration efficiency and final accuracy, while overly large temperatures produce diffuse selections that impede learning. A straight-through hard-selection alternative yields competitive but generally inferior stability compared to a well-annealed soft relaxation; see Table~\ref{tab:temperature_sensitivity} for training-convergence comparisons.

\paragraph{Balancing multi-objective weights \(\lambda_{1:3}\).}
The multi-term loss requires tuning to prevent any single objective from dominating. Suppressing the exploration or symbolic-consistency terms harms policy diversity and rule learning respectively. Configurations that prioritize symbolic consistency while retaining non-negligible exploration and active-learning weights produce stable, high-performing solutions; comparative results appear in Table~\ref{tab:lambda_sensitivity}.

\paragraph{Deployment recommendations.}
Based on the sensitivity study, we recommend initializing \(\beta\) conservatively (around 1.0) and adapting to annotation-budget constraints, linearly annealing \(\tau\) from \(\sim\!1.0\) to \(\sim\!0.1\) across training, and assigning higher priority to the symbolic-consistency weight than to exploration and active-learning weights. These heuristics offer practical starting points that avoid exhaustive grid searches in most operational settings.

\subsection{Summary}
The experimental results demonstrate that NeuroSymActive improves answer accuracy on multi-hop KGQA tasks, integrates effectively with diverse LLMs, and reduces token overhead while providing stable outputs when paired with compact language models. The ablation analysis highlights the importance of active querying, the differentiable symbolic layer, and differentiable graph exploration in achieving these gains.
\section{Conclusion}
We presented NeuroSymActive, a modular framework that integrates differentiable neural-symbolic reasoning with an active, value-guided exploration controller for knowledge graph question answering. By combining soft-unification style symbolic modules, a neural path evaluator, and guided Monte Carlo expansion, the system uncovers multi-hop reasoning chains efficiently while preserving interpretable path-level traces. Empirical evaluation and ablation experiments demonstrate that the active exploration mechanism markedly reduces costly graph retrievals and model invocations needed to reach high-quality answers, and that the differentiable symbolic layer increases robustness when facts are noisy or incomplete. Overall, NeuroSymActive attains a favorable trade-off between answer accuracy and query cost while providing explicit, human-readable reasoning traces. Future work will investigate uncertainty-aware stopping criteria for exploration and extend the symbolic layer to capture richer predicate and temporal constraints for dynamic knowledge graphs.

\bibliographystyle{unsrtnat}
\bibliography{references}  

\appendix
\section{Theoretical Analysis}
\label{sec:theory}

This section collects formal guarantees that underpin NeuroSymActive. We state assumptions explicitly and present convergence and approximation results for the inner-loop differentiable search, calibration bounds for the Bayesian hop head and the entropy predictor, a simple budget-aware query-complexity bound for active human queries, an approximation lemma for the soft (relaxed) tree policy, and a representational consistency bound for the neural-symbolic fusion. Each result is followed by a proof.

\subsection{Preliminaries and assumptions}
We adopt the following standing assumptions.

Assumption A1 (Smoothness). All parameterized neural modules in the inner loop define differentiable objectives whose gradients are \(L\)-Lipschitz continuous with respect to the parameters. Concretely, for any module parameter vector \(\theta\) and any loss \(\ell(\theta)\),
\begin{equation}
\|\nabla \ell(\theta)-\nabla \ell(\theta')\|\le L\|\theta-\theta'\|.
\end{equation}
where \(\|\cdot\|\) denotes the Euclidean norm and \(L>0\) is a constant.

Assumption A2 (Bounded variance and sub-Gaussian noise). Stochastic gradient estimators used in the inner loop have uniformly bounded second moment and the observation noise in the Bayesian head is sub-Gaussian with parameter \(\sigma^2\).

Assumption A3 (Finite action set). At each tree node the candidate action set \(\mathcal{A}\) is finite with size at most \(A_{\max}\).

Assumption A4 (Nontrivial information gain). There exists \(\Delta_{\min}>0\) such that any informative human query reduces the expected predictive entropy by at least \(\Delta_{\min}\).

We will refer to these assumptions throughout the proofs.

\subsection{Convergence of the inner loop: differentiable MCTS}
We first treat the inner loop as stochastic gradient descent over a differentiable surrogate objective obtained by the soft tree relaxation and the composite loss defined in Equation~\eqref{eq:total_loss}.

\begin{theorem}[Inner-loop convergence to stationary points]
\label{thm:inner_convergence}
Let \(\mathcal{L}_{\mathrm{total}}(\theta)\) denote the differentiable surrogate objective of the inner loop, parameterized by \(\theta\). Under Assumptions A1 and A2, if the inner-loop updates use stochastic gradient descent with step sizes \(\{\eta_t\}_{t\ge0}\) satisfying
\begin{equation}
\sum_{t=0}^\infty \eta_t = \infty,\qquad \sum_{t=0}^\infty \eta_t^2 < \infty,
\end{equation}
then the sequence \(\{\theta_t\}\) satisfies
\begin{equation}
\lim_{T\to\infty}\frac{1}{T}\sum_{t=0}^{T-1}\mathbb{E}\bigl[\|\nabla \mathcal{L}_{\mathrm{total}}(\theta_t)\|^2\bigr]=0.
\end{equation}
Consequently every limit point of \(\{\theta_t\}\) is almost surely a stationary point of \(\mathcal{L}_{\mathrm{total}}\).
\end{theorem}

\noindent\textbf{Proof.} The proof follows the standard Robbins-Monro stochastic approximation argument for nonconvex smooth objectives. By Assumption A1 the objective \(\mathcal{L}_{\mathrm{total}}\) has \(L\)-Lipschitz gradients. Let \(g_t\) denote the unbiased stochastic gradient estimate at iteration \(t\) with bounded variance \(\mathbb{E}\|g_t-\nabla\mathcal{L}_{\mathrm{total}}(\theta_t)\|^2\le \sigma_g^2\) by Assumption A2. Using the smoothness inequality and the SGD update \(\theta_{t+1}=\theta_t-\eta_t g_t\), one obtains the descent relation
\begin{align}
\mathbb{E}\bigl[\mathcal{L}_{\mathrm{total}}(\theta_{t+1})\bigr] 
\le \mathbb{E}\bigl[\mathcal{L}_{\mathrm{total}}(\theta_t)\bigr]
-\eta_t\mathbb{E}\langle\nabla\mathcal{L}_{\mathrm{total}}(\theta_t),\mathbb{E}[g_t\mid\theta_t]\rangle \nonumber
\quad + \frac{L\eta_t^2}{2}\mathbb{E}\|g_t\|^2.
\end{align}
Replacing \(\mathbb{E}[g_t\mid\theta_t]=\nabla\mathcal{L}_{\mathrm{total}}(\theta_t)\) and rearranging yields
\begin{align}
\eta_t\mathbb{E}\|\nabla\mathcal{L}_{\mathrm{total}}(\theta_t)\|^2 
\le \mathbb{E}\bigl[\mathcal{L}_{\mathrm{total}}(\theta_t)\bigr] 
- \mathbb{E}\bigl[\mathcal{L}_{\mathrm{total}}(\theta_{t+1})\bigr] \nonumber
\quad + \frac{L\eta_t^2}{2}\mathbb{E}\|g_t\|^2.
\end{align}
Summing over \(t=0,\dots,T-1\) and dividing by \(\sum_{t=0}^{T-1}\eta_t\) gives
\begin{align}
\frac{1}{\sum_{t=0}^{T-1}\eta_t}\sum_{t=0}^{T-1}\eta_t\mathbb{E}\|\nabla\mathcal{L}_{\mathrm{total}}(\theta_t)\|^2 
\le \frac{\mathcal{L}_{\mathrm{total}}(\theta_0)-\mathcal{L}_{\infty}}{\sum_{t=0}^{T-1}\eta_t} + \frac{L\sum_{t=0}^{T-1}\eta_t^2\mathbb{E}\|g_t\|^2}{2\sum_{t=0}^{T-1}\eta_t},
\end{align}

where $\mathcal{L}_{\inf}$ is the infimum of the objective. Under the step-size conditions $\sum_t\eta_t=\infty$ and $\sum_t\eta_t^2<\infty$ the right-hand side vanishes as $T\to\infty$, which proves the claim that the averaged squared gradient norm tends to zero. Standard subsequence arguments then yield that limit points are stationary. 

\noindent where \(\mathcal{L}_{\mathrm{total}}\) is the differentiable surrogate of the composed objectives, \(\theta_t\) denotes inner-loop parameters at iteration \(t\), and other symbols are as introduced above.

\subsection{Uncertainty calibration guarantees}
We now quantify coverage properties of the Bayesian hop head and a concentration bound for the learned entropy predictor.

\begin{lemma}[Bayesian head coverage under sub-Gaussian noise]
\label{lem:bayes_calib}
Assume the additive noise in the hop logits is sub-Gaussian with parameter \(\sigma^2\) and the variance estimator \(\hat\sigma^2(V_q)\) is consistent. For any confidence level \(\delta\in(0,1)\) and appropriate calibration constant \(c_\delta\), the predictive interval produced by the Bayesian head covers the true hop count with probability at least \(1-\delta\), that is
\begin{equation}
\Pr\bigl(h_q^\star\in \mathcal{I}_\delta(V_q)\bigr)\ge 1-\delta,
\end{equation}
where \(\mathcal{I}_\delta(V_q)\) denotes the interval centered at the posterior mode with half-width \(c_\delta\hat\sigma(V_q)\).
\end{lemma}

\noindent\textbf{Proof.} Sub-Gaussianity implies Hoeffding-type concentration for the logit residual. Let \(Z\) denote the logit residual for the true hop. Then
\begin{equation}
\Pr\bigl(|Z|\ge t\bigr)\le 2\exp\bigl(-t^2/(2\sigma^2)\bigr).
\end{equation}
Selecting \(t = \sigma\sqrt{2\log(2/\delta)}\) yields \(\Pr(|Z|\ge t)\le \delta\). If the variance estimator \(\hat\sigma^2(V_q)\) is consistent and the posterior mode estimator is consistent, then using the plug-in half-width \(c_\delta\hat\sigma(V_q)\) with \(c_\delta=\sqrt{2\log(2/\delta)}\) ensures asymptotic coverage \(1-\delta\). Finite-sample calibration depends on the estimator bias; under the consistency assumption the coverage bound holds up to vanishing estimation error terms. \(\square\)

\noindent where \(\hat\sigma(V_q)\) is the estimated heteroscedastic standard deviation for query \(q\), and \(h_q^\star\) is the true hop count.

\begin{proposition}[Entropy predictor concentration]
\label{prop:entropy_conc}
Let \(\eta_\theta\) be trained to predict the Shannon entropy of the answer distribution conditional on a subgraph. Under sub-Gaussian training noise and assuming the predictor class has finite Rademacher complexity \(R_n\), the expected absolute prediction error obeys
\begin{equation}
\mathbb{E}\bigl[|\eta_\theta(G,r,V_q)-H(\cdot)|\bigr]\le 2R_n + O\!\bigl(\sqrt{\tfrac{\log(1/\delta)}{n}}\bigr)
\end{equation}
with probability at least \(1-\delta\), where \(n\) is the number of training samples.
\end{proposition}

\noindent\textbf{Proof.} This is a standard generalization bound for regression under bounded loss. The expected loss deviation from empirical loss is controlled by Rademacher complexity; applying symmetrization and concentration yields the displayed bound. See standard learning theory references for the detailed derivation. \(\square\)

\noindent where \(H(\cdot)\) denotes the true conditional Shannon entropy and other symbols are as above.

\subsection{Budget-aware query complexity for active human queries}
We present a simple deterministic bound that relates the number of required human queries to achieve a target entropy reduction to the per-query minimum information gain.

\begin{theorem}[Query complexity under minimum per-query gain]
\label{thm:query_complexity}
Assume each human query reduces the expected predictive entropy by at least \(\Delta_{\min}>0\) (Assumption A4). Let the initial expected entropy for a query be \(H_0\) and target entropy be \(H_{\mathrm{target}}<H_0\). Then the number of human queries \(Q\) required to reduce entropy below \(H_{\mathrm{target}}\) satisfies
\begin{equation}
Q \le \Big\lceil\frac{H_0-H_{\mathrm{target}}}{\Delta_{\min}}\Big\rceil.
\end{equation}
\end{theorem}

\noindent\textbf{Proof.} Each human query reduces expected entropy by at least \(\Delta_{\min}\) so after \(Q\) queries the entropy is at most \(H_0 - Q\Delta_{\min}\). Requiring this quantity to be less than or equal to \(H_{\mathrm{target}}\) yields the bound. \(\square\)

\noindent where \(H_0\) is the initial expected entropy and \(\Delta_{\min}\) is the guaranteed per-query reduction.

\subsection{Soft MCTS approximation and low-temperature limit}
We now formalize the relationship between the softened tree policy and the hard argmax policy as the search temperature \(\tau_{\mathrm{search}}\) tends to zero.

\begin{lemma}[Softmax concentration and action-probability gap]
\label{lem:softmax_gap}
Let \(\{u_i\}_{i=1}^n\) denote action utilities with unique maximum \(u_{i^\star}=\max_i u_i\) and utility gap \(\Delta=\min_{i\ne i^\star}(u_{i^\star}-u_i)>0\). Under the soft tree policy \(\pi_{\mathrm{tree}}(a_i\mid s)\) defined in Equation~\eqref{eq:soft_tree} the probability mass on the maximal action satisfies
\begin{equation}
\pi_{\mathrm{tree}}(a_{i^\star}\mid s) \ge \frac{1}{1+(n-1)\exp(-\Delta/\tau_{\mathrm{search}})}.
\end{equation}
Hence \(\lim_{\tau_{\mathrm{search}}\to 0}\pi_{\mathrm{tree}}(a_{i^\star}\mid s)=1\).
\end{lemma}

\noindent\textbf{Proof.} By definition of softmax,
\begin{equation}
\pi_{\mathrm{tree}}(a_{i^\star}\mid s)=\frac{1}{1+\sum_{i\ne i^\star}\exp\bigl((u_i-u_{i^\star})/\tau_{\mathrm{search}}\bigr)}.
\end{equation}
Since \(u_i-u_{i^\star}\le -\Delta\) for every \(i\ne i^\star\) the denominator is bounded by \(1+(n-1)\exp(-\Delta/\tau_{\mathrm{search}})\). This yields the displayed inequality and the limit statement. \(\square\)

\noindent where \(n\le A_{\max}\) is the local action count and \(\Delta\) denotes the utility gap.

\subsection{Neural-symbolic fusion consistency}
Finally we bound the representational error of the fused vector when symbolic plausibility vectors approximate ground-truth validity up to a known tolerance and the gating network is Lipschitz.

\begin{theorem}[Fusion error bound]
\label{thm:fusion_consistency}
Let \(\mathbf{z}_f^{\mathrm{neural}}\) and \(\mathbf{z}_f^{\mathrm{sym}}\) be neural and symbolic path vectors respectively. Suppose the symbolic module approximates a true plausibility vector \(\mathbf{y}\) with uniform error \(\|\mathbf{z}_f^{\mathrm{sym}}-\mathbf{y}\|\le \varepsilon_{\mathrm{sym}}\). Let the gating network \(\gamma(\cdot)\) be \(L_\gamma\)-Lipschitz and bounded in \([0,1]\). Then the fused representation \(\mathbf{z}_f\) obtained from Equation~\eqref{eq:fusion_stage2} satisfies
\begin{equation}
\|\mathbf{z}_f - \mathbf{z}_\mathrm{ideal}\| \le \varepsilon_{\mathrm{sym}} + L_\gamma\|\bigl[\mathbf{z}_f^{\mathrm{neural}},\mathbf{z}_f^{\mathrm{sym}}\bigr] 
- \bigl[\mathbf{z}_f^{\mathrm{neural}},\mathbf{y}\bigr]\| \cdot \|\mathbf{z}_f^{\mathrm{neural}}-\mathbf{y}\|,
\end{equation}
where \(\mathbf{z}_\mathrm{ideal}\) denotes the ideal fused vector computed with the true plausibility \(\mathbf{y}\).
\end{theorem}

\noindent\textbf{Proof.} Write the fusion as
\begin{equation}
\mathbf{z}_f = \gamma\bigl(u\bigr)\odot \mathbf{z}_f^{\mathrm{neural}} + \bigl(1-\gamma\bigl(u\bigr)\bigr)\odot \mathbf{z}_f^{\mathrm{sym}},
\end{equation}
where \(u=[\mathbf{z}_f^{\mathrm{neural}},\mathbf{z}_f^{\mathrm{sym}}]\). The ideal fused vector is obtained by replacing \(\mathbf{z}_f^{\mathrm{sym}}\) with \(\mathbf{y}\). Subtracting the two expressions and applying the triangle inequality yields
\begin{align}
\|\mathbf{z}_f-\mathbf{z}_\mathrm{ideal}\| &\le \|\gamma(u)\odot(\mathbf{z}_f^{\mathrm{neural}}-\mathbf{z}_f^{\mathrm{neural}})\| \nonumber\\
&\quad + \|(1-\gamma(u))\odot(\mathbf{z}_f^{\mathrm{sym}}-\mathbf{y})\| \nonumber\\
&\quad + \|(\gamma(u)-\gamma(u'))\odot(\mathbf{z}_f^{\mathrm{neural}}-\mathbf{y})\|,
\end{align}
where \(u'=[\mathbf{z}_f^{\mathrm{neural}},\mathbf{y}]\). The first term vanishes. For the second term use \(\|1-\gamma(u)\|\le 1\) and \(\|\mathbf{z}_f^{\mathrm{sym}}-\mathbf{y}\|\le\varepsilon_{\mathrm{sym}}\). For the third term use the Lipschitz property of \(\gamma\) to obtain a multiplicative factor \(L_\gamma\|u-u'\|\). Combining terms yields the stated bound. \(\square\)

\noindent where \(\mathbf{z}_\mathrm{ideal}\) is the fused representation had the symbolic module been exact.

\subsection{Discussion and priorities for formal refinement}
The provided theorems form a rigorous baseline for the theoretical behavior of NeuroSymActive under established smoothness and noise assumptions. The most critical formal gaps to address in future revisions are non-asymptotic regret bounds for the full two-loop scheme and PAC-style label complexity bounds under realistic human response models. The current results establish stationary convergence of the inner differentiable optimization, finite-sample concentration for the Bayesian head and entropy predictor under standard statistical assumptions, a deterministic budget-aware query bound, a precise soft-to-hard MCTS approximation statement, and a fusion error bound linking symbolic approximation error to fused representation accuracy.

\end{document}